\documentclass[letterpaper,conference]{IEEEtran} 

\usepackage[utf8]{inputenc} \usepackage[T1]{fontenc}    \usepackage[colorlinks,linkcolor=blue,citecolor=blue,bookmarks=false,pagebackref]{hyperref}            \usepackage{url}            \usepackage{booktabs}       \usepackage{amsfonts}       \usepackage{nicefrac}       \usepackage{microtype}      \usepackage{algorithm,algorithmic}
\usepackage{etoolbox}

\usepackage{cite}

\usepackage{amsmath,amssymb,mathtools,mathrsfs}

\newtoggle{conference}
\togglefalse{conference} 

\iftoggle{conference}{}{
\pagenumbering{arabic}
\usepackage{fancyhdr} 
\fancyhf{}
\cfoot{\thepage}
\pagestyle{fancy}  
 }

\usepackage{xcolor}

\usepackage{graphicx}

\usepackage{tikz}
\usetikzlibrary{arrows}

\makeatletter
\newcommand{\pushright}[1]{\ifmeasuring@#1\else\omit\hfill$\displaystyle#1$\fi\ignorespaces}
\newcommand{\pushleft}[1]{\ifmeasuring@#1\else\omit$\displaystyle#1$\hfill\fi\ignorespaces}
\makeatother

\usepackage[utf8]{inputenc}
\usepackage[T1]{fontenc}
\usepackage{url}
\usepackage{ifthen}

\usepackage{amsmath} 

\usepackage{hyperref}       \usepackage{booktabs}       \usepackage{amsfonts}       \usepackage{nicefrac}       \usepackage{microtype}      

\usepackage{amssymb, bm, epsfig, psfrag,amsthm}
\usepackage{algorithm,algorithmic}
\usepackage{bbm}
\usepackage{color}
\usepackage{tikz}
\usepackage{enumerate}
\usetikzlibrary{shapes,arrows}
\usetikzlibrary{decorations.pathreplacing,calc,shapes.misc,shapes.callouts,shapes.arrows,positioning}
\usetikzlibrary{patterns}
\usepackage{url,hyperref}
\usepackage{xspace}
\usepackage{enumerate}
\usepackage{mathtools}
\usepackage{epstopdf}
\usepackage{sidecap}
\usepackage{graphicx}
\graphicspath{{figures/}}

\usepackage[normalem]{ulem}

\renewcommand{\hat}{\widehat}

\def\beq{\begin{equation}}
\def\eeq{\end{equation}}
\def\beqa{\begin{eqnarray}}
\def\eeqa{\end{eqnarray}}
\def\beqan{\begin{eqnarray*}}
\def\eeqan{\end{eqnarray*}}

\def\R{{\mathbb{R}}}

\DeclareMathOperator*{\argmin}{arg\,min}
\DeclareMathOperator*{\argmax}{arg\,max}
\DeclareMathOperator{\diag}{diag}

\def\x{\bm{x}}

\newtheorem{theorem}{Theorem}
\newtheorem{lemma}{Lemma}

\theoremstyle{definition}

\newtheorem{definition}{Definition}
\newtheorem{assumption}{Assumption}

\def\arr{\rightarrow}
\def\Exp{\mathbb{E}}

\def\Cov{\mathrm{Cov}}
\def\Tr{\mathrm{Tr}}

\def\PL2{\stackrel{PL(2)}{=}}

\def\km{k\! - \!}
\def\kp{k\! + \!}
\def\lp{\ell\! + \!}
\def\lm{\ell\! - \!}
\def\Lm{L\! - \!}

\newcommand{\zero}{\mathbf{0}}

\newcommand{\bbf}{\mathbf{b}}

\newcommand{\fbf}{\mathbf{f}}
\newcommand{\gbf}{\mathbf{g}}

\newcommand{\pbf}{\mathbf{p}}

\newcommand{\qbf}{\mathbf{q}}

\newcommand{\rbf}{\mathbf{r}}

\newcommand{\ubf}{\mathbf{u}}

\newcommand{\vbf}{\mathbf{v}}

\newcommand{\wbf}{\mathbf{w}}

\newcommand{\xbf}{\mathbf{x}}

\newcommand{\ybf}{\mathbf{y}}

\newcommand{\zbf}{\mathbf{z}}

\newcommand{\zbfhat}{\widehat{\mathbf{z}}}
\newcommand{\Abf}{\mathbf{A}}
\newcommand{\Bbf}{\mathbf{B}}

\newcommand{\Gbf}{\mathbf{G}}

\newcommand{\Hbf}{\mathbf{H}}

\newcommand{\Ibf}{\mathbf{I}}

\newcommand{\Kbf}{\mathbf{K}}

\newcommand{\Pbf}{\mathbf{P}}

\newcommand{\Qbf}{\mathbf{Q}}
\newcommand{\Rbf}{\mathbf{R}}

\newcommand{\Sbf}{\mathbf{S}}
\newcommand{\Ubf}{\mathbf{U}}

\newcommand{\Vbf}{\mathbf{V}}

\newcommand{\Wbf}{\mathbf{W}}

\newcommand{\Xbf}{\mathbf{X}}
\newcommand{\Ybf}{\mathbf{Y}}
\newcommand{\Zbf}{\mathbf{Z}}

\newcommand{\Zbfhat}{\wh{\mathbf{Z}}}
\newcommand{\inner}[1]{\langle{#1}\rangle}
\def\alphabf{{\boldsymbol \alpha}}
\def\betabf{{\boldsymbol \beta}}

\def\Gammabf{{\boldsymbol \Gamma}}

\def\lambdabar{\overline{\lambda}}
\newcommand{\Lambdabf}{\mathbf{\Lambda}}

\def\Upsilonbar{\overline{\Upsilon}}
\def\Lambdabfbar{\overline{\mathbf{\Lambda}}}
\def\Omegabf{\mathbf{\Omega}}
\def\Omegabfbar{\overline{\mathbf{\Omega}}}

\def\Gammabfbar{\overline{\mathbf{\Gamma}}}
\def\Thetabfbar{\overline{\mathbf{\Theta}}}

\def\xibf{{\boldsymbol \xi}}
\def\Xibf{{\boldsymbol \Xi}}

\def\msf{\mathsf}
\def\taubf{{\boldsymbol \tau}}

\def\varphibf{{\boldsymbol \varphi}}

\newcommand{\Thetabf}{{\bm{\Theta}}}

\newcommand{\phibf}{{\bm{\phi}}}

\newcommand{\mubar}{\overline{\mu}}

\newcommand{\tran}{^{\text{\sf T}}}

\def\eqd{\stackrel{d}{=}}

\def\Norm{{\mathcal N}}
\def\Range{\mathrm{Range}}

\newcommand*\dif{\mathop{}\!\mathrm{d}}  

\newcommand{\bkt}[1]{{\left< #1 \right>}}
 \def\Gset{\mathfrak{G}}
\def\Gsetbar{\overline{\mathfrak{G}}}

\RequirePackage{xcolor,bm,setspace}
\RequirePackage{mathrsfs}

\providecommand{\old}[1]{ }

\providecommand{\mc}{\mathcal}
\providecommand{\ie}{\rm i.e.}
\providecommand{\T}{^\top}
\providecommand{\wb}{\overline}
\providecommand{\wt}{\widetilde}
\providecommand{\wh}{\widehat}
\newcommand{\norm}[1]{\left\|#1\right\|}
\providecommand{\Real}{\mathbb{R}}

\providecommand{\i}{\bm{i}}
\providecommand{\j}{\bm{j}}

\providecommand{\p}{\mathbf{p}}
\providecommand{\q}{\mathbf{q}}

\providecommand{\w}{\mathbf{w}}
\providecommand{\x}{\bm{x}}
\providecommand{\y}{\bm{y}}

\providecommand{\M}{\bm{M}}

\providecommand{\V}{\mathbf{V}}

\providecommand{\X}{\bm{X}}

\providecommand{\bbf}{\mathbf{b}}

\providecommand{\fbf}{\mathbf{f}}
\providecommand{\gbf}{\mathbf{g}}

\providecommand{\qbf}{\mathbf{q}}
\providecommand{\rbf}{\mathbf{r}}

\providecommand{\ubf}{\mathbf{u}}
\providecommand{\vbf}{\mathbf{v}}
\providecommand{\wbf}{\mathbf{w}}
\providecommand{\xbf}{\mathbf{x}}
\providecommand{\ybf}{\mathbf{y}}
\providecommand{\zbf}{\mathbf{z}}
\providecommand{\Abf}{\mathbf{A}}
\providecommand{\Bbf}{\mathbf{B}}

\providecommand{\Fbf}{\mathbf{F}}
\providecommand{\Gbf}{\mathbf{G}}
\providecommand{\Hbf}{\mathbf{H}}
\providecommand{\Ibf}{\mathbf{I}}

\providecommand{\Kbf}{\mathbf{K}}

\providecommand{\Nbf}{\mathbf{N}}

\providecommand{\Pbf}{\mathbf{P}}
\providecommand{\Qbf}{\mathbf{Q}}
\providecommand{\Rbf}{\mathbf{R}}
\providecommand{\Sbf}{\mathbf{S}}

\providecommand{\Ubf}{\mathbf{U}}
\providecommand{\Vbf}{\mathbf{V}}
\providecommand{\Wbf}{\mathbf{W}}
\providecommand{\Xbf}{\mathbf{X}}
\providecommand{\Ybf}{\mathbf{Y}}
\providecommand{\Zbf}{\mathbf{Z}}

\providecommand{\phibf}{\mbf{\phi}}

\providecommand{\mcN}{\mathcal{N}}

\def\Lodd{\mc L_{\rm odd}}
\def\Leven{\mc L_{\rm even}}

\title{Inference in Multi-Layer Networks with Matrix-Valued Unknowns}
\author{
\IEEEauthorblockN{Parthe Pandit,$^{1}$ 
Mojtaba Sahraee-Ardakan,$^{1}$
Sundeep Rangan,$^{2}$ Philip Schniter,$^{3}$
and Alsyon K. Fletcher$^{1}$} 

\IEEEauthorblockA{$^1$Dept.\ Statistics and ECE, University of California, Los Angeles}

\IEEEauthorblockA{$^2$Dept.\ ECE, New York University, Brooklyn, NY}

\IEEEauthorblockA{$^3$Dept.\ ECE, The Ohio State University, Columbus, Ohio}
}
\usepackage{scalerel}
\begin{document}

\maketitle
\setlength{\abovedisplayskip}{7pt}
\setlength{\belowdisplayskip}{7pt}
\setlength{\parskip}{3pt}
\begin{abstract}
We consider the problem of inferring the input and hidden variables of a stochastic multi-layer 
neural network from an observation of the output.  
The hidden variables in each layer are represented as matrices.
This problem applies to signal recovery via deep generative prior models,
multi-task and mixed regression, and
learning certain classes of two-layer neural networks. 
A unified approximation algorithm for both MAP and MMSE inference is proposed by extending 
a recently-developed Multi-Layer Vector Approximate Message Passing (ML-VAMP) algorithm to handle 
matrix-valued unknowns.
It is shown that the performance of the proposed Multi-Layer Matrix VAMP (ML-Mat-VAMP) algorithm
can be exactly predicted in a certain random large-system limit, 
where the dimensions $N\times d$ of the unknown quantities grow as $N\rightarrow\infty$ with $d$ fixed.
In the two-layer neural-network learning problem, this scaling corresponds to the case where the number of input features and training samples grow to infinity but the number of hidden nodes stays fixed.
The analysis enables a precise prediction of the parameter and test error of the learning.
\end{abstract}

\section{Introduction}

\begin{figure*}[t]
    \centering
    \begin{tikzpicture}

    \pgfmathsetmacro{\sep}{3};
    \pgfmathsetmacro{\yoff}{0.4};
    \pgfmathsetmacro{\xoffa}{0.3};
    \pgfmathsetmacro{\xoffb}{0.6};

    \tikzstyle{var}=[draw,circle,fill=green!20,node distance=2.5cm];
    \tikzstyle{yvar}=[draw,circle,fill=orange!30,node distance=2.5cm];
    \tikzstyle{conn}=[draw,circle,fill=green!40,radius=0.02cm];
    \tikzstyle{linest}=[draw,fill=blue!20,minimum size=1cm,
        minimum height=1.8cm, node distance=2cm]
    \tikzstyle{nlest}=[draw,fill=green!20,minimum size=1cm,
        minimum height=1.8cm, node distance=2cm]
    \tikzstyle{linblock}=[draw,fill=blue!20, minimum size=1cm, node distance=\sep cm];
    \tikzstyle{nlblock}=[draw,fill=green!20, minimum size=1cm, node distance=\sep cm];

\node [var] (z0) {};
    \node [linblock, right of=z0] (W1) {$\Wbf_1,\bbf_1$};
    \node [nlblock, right of=W1] (phi2) {$\phibf_2(\cdot)$};
    \node [linblock, right of=phi2] (W3) {$\Wbf_3,\bbf_3$};
    \node [nlblock, right of=W3] (phi4) {$\phibf_4(\cdot)$};
    \node [yvar,right of=phi4, node distance=2cm] (y) {};
\node [var, above of=W1, yshift=-1cm] (xi1) {};
    \node [var, above of=phi2, yshift=-1cm] (xi2) {};
    \node [var, above of=W3, yshift=-1cm] (xi3) {};
    \node [var, above of=phi4, yshift=-1cm] (xi4) {};
\path[->] (z0) edge  node [above] {$\Zbf^0_0$} (W1);
    \path[->] (W1) edge  node [above] {$\Zbf^0_1$} (phi2);
    \path[->] (phi2) edge  node [above] {$\Zbf^0_2$} (W3);
    \path[->] (W3) edge  node [above] {$\Zbf^0_3$} (phi4);
    \path[->] (phi4) edge  node [above] {$\Ybf$} (y);
\path[->] (xi1) edge node [right] {\small$\Xibf_1$} (W1);
    \path[->] (xi2) edge node [right] {\small$\Xibf_2$} (phi2);
    \path[->] (xi3) edge node [right] {\small$\Xibf_3$} (W3);
    \path[->] (xi4) edge node [right] {\small$\Xibf_4$} (phi4);

\node [nlest,below of=z0] (h0) {$\Gbf^+_0(\cdot)$};
    \node [linest,below of=W1] (h1) {$\Gbf^{\pm}_1(\cdot)$};
    \node [nlest,below of=phi2] (h2) {$\Gbf^{\pm}_2(\cdot)$};
    \node [linest,below of=W3] (h3) {$\Gbf^{\pm}_3(\cdot)$};
    \node [nlest,below of=phi4] (h4) {$\Gbf^-_4(\cdot)$};
    \node [yvar,right of=h4, node distance=2cm] (y2) {};
    \path [draw,->] (y2) edge node [above] {$\Ybf$} (h4.east);

    \foreach \i/\j in {0/1,1/2,2/3,3/4} {
        \node [right of=h\i,xshift=\xoffa cm,yshift=\yoff cm] (conn0\i) {};
        \node [right of=h\i,xshift=\xoffb cm,yshift=\yoff cm] (conn1\i) {};
        \node [right of=h\i,xshift=\xoffa cm,yshift=-\yoff cm] (conn2\i) {};
        \node [right of=h\i,xshift=\xoffb cm,yshift=-\yoff cm] (conn3\i) {};
        \draw [fill=blue!20] (conn1\i) circle (0.1cm);
        \draw [fill=blue!20] (conn2\i) circle (0.1cm);

        \path [draw,->] (conn0\i) -- (conn2\i);
        \path [draw,->] (conn3\i) -- (conn1\i);

        \path[->] ([yshift=\yoff cm]h\i.east) edge node  [above]
            {$\widehat{\Zbf}^+_{k\i}$} (conn1\i);
        \path[->] (conn1\i) edge node  [above]
            {$\Rbf^+_{k\i}$} ([yshift=\yoff cm]h\j.west);
        \path[->] ([yshift=-\yoff cm]h\j.west) edge node [below]
            {$\widehat{\Zbf}^-_{k\i}$} (conn2\i);
        \path[->] (conn2\i) edge node  [below]
            {$~\Rbf^-_{k\i}$} ([yshift=-\yoff cm]h\i.east);

    }

\end{tikzpicture}     \caption{(TOP) The signal flow graph for \textit{true} values of matrix variables $\{\Zbf_\ell^0\}_{\ell=0}^3$, given in eqn. \eqref{eq:nntrue} where $\Zbf_\ell^0\in \Real^{n_\ell\times d}.$
    (BOTTOM) Signal flow graph of the ML-MVAMP procedure in Algo. \ref{algo:ml-mat-vamp}. The variables with superscript + and - are updated in the forward and backward pass respectively. ML-MVAMP (Algorithm~\ref{algo:ml-mat-vamp}) solves \eqref{eq:main_problem} by solving a sequence of simpler estimation problems over consecutive pairs $(\Zbf_\ell,\Zbf_{\ell-1})$.}
    \label{fig:mlmamp_signal_flow}
\end{figure*}
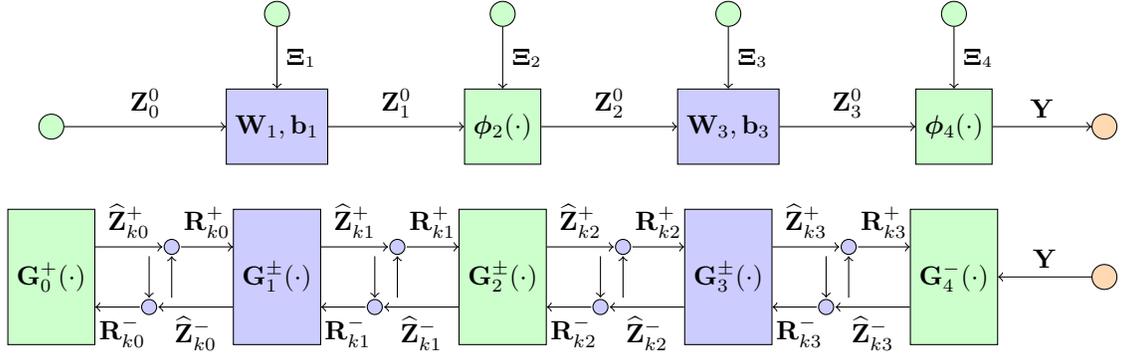

Consider an $L$-layer stochastic neural network given by
\begin{subequations}
\label{eq:nntrue}
\begin{align}
    \Zbf^0_\ell &= \Wbf_{\ell}\Zbf^0_{\lm1} + \Bbf_{\ell} + \Xibf_{\ell}^0,
    \quad &\ell&=1,3,\ldots,\Lm1,
        \label{eq:nnlintrue} \\
    \Zbf^0_{\ell} &= \phibf_\ell(\Zbf^0_{\lm1},\Xibf_{\ell}^0), \quad &\ell& = 2,4,\ldots,L,
        \label{eq:nnnonlintrue} 
\end{align}
\end{subequations}
where, for $\ell=0,1,\ldots,L$, we have 
\textit{true} activations $\Zbf_{\ell}^0\in\Real^{n_\ell\times d}$, 
weights $\Wbf_{\ell}\in \Real^{n_{\ell}\times n_{\ell-1}}$, 
biases $\Bbf_\ell \in\Real^{n_\ell\times d}$, 
and true noise realizations $\Xibf_{\ell}^0$. 
The activation functions $\phibf_\ell$ are known, non-linear functions acting row-wise on their inputs. 
See Fig.~\ref{fig:mlmamp_signal_flow} (TOP). 
We use the superscript $^0$ in $\Zbf_\ell^0$ to indicate
the true values of the variables, in contrast to estimated values discussed later.
We model the true values $\Zbf^0_0$ 
as a realization of random $\Zbf_0$, 
where the rows $\zbf_{0,i:}\tran$ of $\Zbf_0$ are i.i.d.\ with distribution $p_0$: 
\begin{equation}
    p(\Zbf_0)=\prod_{i=1}^{n_0} p_0(\zbf_{0,i:}) \label{eq:input_prior}.
\end{equation}
Similarly, we also assume that $\Xibf_\ell^0$ are realizations of random $\Xibf_\ell$ with i.i.d.\ rows $\xibf_{\ell,i:}\tran$.
For odd $\ell$, the rows $\xibf_{\ell,i:}$ are zero-mean multivariate Gaussian with covariance $\Nbf_{\ell}^{-1}\in\Real^{d\times d}$, 
whereas for even $\ell$, the rows $\xibf_{\ell,i:}$ can be arbitrarily distributed but i.i.d.

Denoting by $\Ybf:=\Zbf_L^0\in \Real^{n_L\times d}$ the output of the network, 
we consider the following matrix inference problem:
\begin{equation}\label{eq:main_problem} \begin{aligned}
&{\rm\ Estimate \ } \Zbf:=\{\Zbf_{\ell}\}_{\ell=0}^{L-1}\\
&{\rm\ given\ } \Ybf:=\Zbf^0_L {\rm\ and\ } \{\Wbf_{2k-1},\Bbf_{2k-1},\phibf_{2k}\}_{k=1}^{L/2} .
\end{aligned}
\end{equation}

This inference problem arises 
in reconstruction with deep generative priors \cite{yeh2016semantic,bora2017compressed}:
A deep neural network is trained as a generative model for some complex data, such as an image. 
The generative model could be a variational auto-encoder (VAE) \cite{rezende2014stochastic,kingma2013auto},
generative adversarial network (GAN) \cite{radford2015unsupervised,salakhutdinov2015learning},
or deep image prior (DIP) \cite{ulyanov2018deep,van2018compressed}.
The generative model is driven by some noise-like innovations signal $\Zbf^0_0$.
Subsequent layers are then appended to model a lossy measurement process
such as blurring, occlusion, or noise, resulting in a corrupted output $\Ybf$.
With the addition of the measurement layer(s), the original image or data manifests as one of the hidden variables in the network.
The problem of estimating the original image from the output $\Ybf$ is then equivalent to the inference problem \eqref{eq:main_problem} of estimating the values of a hidden layer in a multi-layer network from its output.
\iftoggle{conference}{}{
Note that, in this application, the problem is \emph{not} the ``learning problem'' for the network,
since the weights and biases of the network are assumed known (i.e., trained).
}

In many applications, the network \eqref{eq:nntrue} has no
noise $\Xibf_\ell$ for all layers except the final one, $\ell=L$.
In this case, the $\Zbf^0_{L-1} = \gbf(\Zbf_0^0)$ for some 
\emph{deterministic function} $\gbf(\cdot)$ representing the action of the first $L-1$ layers.   Inference can then be conducted
via a minimization of the form,
\beq \label{eq:zmin_det}
    \Zbfhat_0 := \argmin_{\Zbf_0}  H_L(\Ybf,\Zbf_{L-1})
    + H_0(\Zbf_0), 
\eeq
where $\Zbf_{L-1} = \gbf(\Zbf_0)$ is the network output, 
the term $H_L(\Ybf,\Zbf_{L-1})$ 
penalizes the prediction error and $H_0(\Zbf_0)$ is an (optional)
regularizer on the network input.  
For maximum a priori (MAP) estimation one takes,
\begin{subequations} \label{eq:det_map}
\begin{align} 
    H_L(\Ybf,\Zbf_{L-1}) &= -\log p(\Ybf|\Zbf_{L-1}), \\
    H_0(\Zbf_0) &= -\log p(\Zbf_0),
\end{align}
\end{subequations}
where the output probability $p(\Ybf|\Zbf_{L-1})$ is defined
from the last layer of model \eqref{eq:nnnonlintrue}: 
$\Ybf=\Zbf_{L} = \phibf_L(\Zbf_{L-1},\Xibf_{L})$.
The minimization \eqref{eq:zmin_det} can then 
be solved using a gradient-based method.  
Encouraging results in image reconstruction 
have been demonstrated in \cite{yeh2016semantic,bora2017compressed,hand2017global,kabkab2018task,shah2018solving,tripathi2018correction,mixon2018sunlayer}.
Markov-chain Monte Carlo (MCMC) algorithms and Langevin diffusion \cite{cheng2018sharp, welling2011bayesian} could also be employed
for more complex inference tasks.
However, rigorous analysis of these methods is difficult due to the non-convex nature
of the optimization problem.
To address this issue, recent works \cite{manoel2018approximate,fletcher2018inference,pandit2019inference} have extended
Approximate Message Passing (AMP) methods to provide inference algorithms for the multi-layer networks.
AMP was originally developed in \cite{DonohoMM:09,DonohoMM:10-ITW1,BayatiM:11} for compressed sensing.
Similar to other AMP-type results, the performance of multi-layer AMP-based inference 
can be precisely characterized in certain high-dimensional random instances.  
In addition, the mean-squared error for inference of the algorithms match predictions for the Bayes-optimal
inference predicted by various techniques from statistical physics  
\cite{reeves2017additivity,gabrie2018entropy,barbier2019optimal}.
Thus, AMP-based multi-layer inference provides a computationally tractable estimation framework with precise
performance guarantees and testable conditions for optimality in certain high-dimensional random settings.

Prior multi-layer AMP works \cite{he2017generalized,manoel2018approximate,fletcher2018inference,pandit2019inference}
have considered the case of vector-valued quantities with $d=1$.
The main contribution of this paper is to consider the \emph{matrix-valued} case when $d>1$.
As we will see in Section~\ref{sec:applications},
the matrix-valued case applies to multi-task and mixed regression
problems, sketched clustering, as well as learning certain classes of two layer networks.

To handle the case when $d>1$,
we extend the Multi-Layer Vector Approximate Message Passing (ML-VAMP) algorithm of
\cite{fletcher2018inference,pandit2019inference} to the matrix case.  
The ML-VAMP method is based on VAMP method of \cite{rangan2019vamp}, which is closely related to 
expectation propagation (EP) \cite{minka2001expectation,takeuchi2017rigorous},
expectation-consistent approximate inference (EC) \cite{opper2005expectation,fletcher2016expectation},
S-AMP \cite{cakmak2014samp}, and orthogonal AMP \cite{ma2017orthogonal}. 
We will use ``ML-Mat-VAMP'' when referring to the matrix extension of ML-VAMP.
 
Similar to the case of ML-VAMP, we analyze ML-Mat-VAMP in a large system limit, 
where $n_\ell \rightarrow \infty$ and $d$ is fixed, 
under rotationally invariant random weight matrices $\Wbf_\ell$.
In this large system limit, we prove that the mean-squared error (MSE) of the estimates of ML-Mat-VAMP
can be exactly predicted by a deterministic set of equations called the \emph{state evolution} (SE).
\iftoggle{conference}{}{
In the case of ML-VAMP, the SE equations involve scalar quantities
and $2 \times 2$ matrices.
For ML-Mat-VAMP, the SE equations involve $d \times d$ and
$2d \times 2d$ matrices.  For learning problems,
we will see that the SE equations enables predictions
of the parameter error as well as test error.
}

\iftoggle{conference}{}{
\subsubsection*{Notation}:
Boldface uppercase letters $\Xbf$ denote matrices. $\Xbf_{n:}$ refers to the $n^{\rm th}$ row of $\Xbf$. Random vectors are row-vectors. For a function $f:\Real^{1\times m}\rightarrow\Real^{1\times n},$ its row-wise extension is represented by $\fbf:\Real^{N\times m}\rightarrow\Real^{N\times n}$, \ie, $[\fbf(\Xbf)]_{n:}=f(\Xbf_{n:})$. We denote the Jacobian matrix of $f$ by $\tfrac{\partial f}{\partial \x}(\x,\y)\in\Real^{m\times n}$, so that $[\tfrac{\partial f}{\partial \x}(\x,\y)]_{ij}=\tfrac{\partial f_i}{\partial \x_j}(\x,\y)$. For its row-wise extension $\fbf$, we denote by $\inner{\tfrac{\partial \fbf}{\partial \Xbf}(\Xbf,\Ybf)}$ the average Jacobian, \ie, $\tfrac1N\sum_{n=1}^N\tfrac{\partial f}{\partial \Xbf_{n:}}(\Xbf_{n:},\Ybf_{n:})$
} 
\section{Example Applications}
\label{sec:applications}

\subsection{Multi-task and Mixed Regression Problems}
\label{sec:mmv_regression}

A simple application of the 
matrix-valued multi-layer inference problem \eqref{eq:main_problem}
is for \emph{multi-task regression} \cite{obozinski2006multi}.
Consider a linear model of the form,
\beq \label{eq:linear_multi_task}
    \Ybf = \Xbf\Fbf + \Xibf,
\eeq
where $\Ybf \in \R^{N \times d}$ is a matrix of measured responses,
$\Xbf \in \R^{N \times p}$ is a known matrix, $\Fbf \in \R^{p \times d}$ are a set regression coefficients to be estimated,
and $\Xibf$ is additive noise.
The problem can be considered as $d$ separate linear regression
problems -- one for each column.  
However, in some applications, these design ``tasks'' are related in such a way that it benefits to \emph{jointly} estimate the predictors.
To do this, it is common to solve an optimization problem of the form
\begin{equation}
    \arg\min_{\Fbf} \Bigg\{ \sum_{j=1}^d \sum_{i=1}^{N} L(y_{ij},[\Xbf\Fbf]_{ij}) 
        + \lambda \sum_{k=1}^p \rho(\Fbf_{k:}) \Bigg\} 
    \label{eq:multi-task},
\end{equation}
where 
$L(\cdot)$ is a loss function, 
and $\rho(\cdot)$ is a regularizer that acts on 
the rows of $\Fbf_{k:}$ of $\Fbf$ to couple the prediction coefficients across tasks.
For example, the multi-task LASSO \cite{obozinski2006multi} uses loss $L(y,z)=(y-z)^2$ and regularization $\rho(\Fbf_{k:})=\|\Fbf_{k:}\|_2$ to enforce row-sparsity in $\Fbf$.
In the compressive-sensing context, multi-task regression is known as the ``multiple measurement vector'' (MMV) problem, with applications in 
MEG reconstruction \cite{cotter2005sparse}, 
DoA estimation \cite{tzagkarakis2010multiple}, and 
parallel MRI \cite{liang2009parallel}.
An AMP approach to the MMV problem was developed in
\cite{ziniel2012efficient}.
The multi-task model \eqref{eq:linear_multi_task}
can be immediately written as a multi-layer network
\eqref{eq:nntrue} by setting:
\begin{align*}
    \Zbf_0 &:= \Fbf, \quad \Wbf_0 := \Xbf, \\
    \Zbf_1 &:= \Wbf_0\Zbf_0 = \Xbf\Fbf, \quad
    \Ybf=\Zbf_2 := \Zbf_1 + \Xibf.
\end{align*}
Also, by appropriately setting the prior $p(\Zbf_0)$, the multi-layer
matrix MAP inference \eqref{eq:det_map} will match the
multi-task optimization~\eqref{eq:multi-task}.
    
In \eqref{eq:multi-task}, the regularization couples the columns of $\Fbf$ but the loss term does not.
In \emph{mixed regression} problems, the loss couples the columns of $\Fbf$.
For example, consider designing predictors $\Fbf=[\fbf_1,\fbf_2]$ for \emph{mixed linear regression} \cite{yi2014alternating}, i.e.,
\beq \label{eq:linear_mixed}
    y_i = q_i \xbf_i\tran\fbf_1 + (1-q_i) \xbf_i\tran\fbf_2 + v_i, ~~ q_i\in\{0,1\}, 
\eeq
where $i=1,\ldots,N$ and the $i$th response comes 
from one of two linear models, but which model is not known.  This setting can be modeled
by a different output mapping:  As before, set $\Zbf_1=\Xbf\Fbf$ 
and let the noise in the output layer be $\Xibf_1 = [\qbf,\vbf]$
which includes the additive noise $v_i$ in \eqref{eq:linear_mixed}
and the random selection variable $q_i$.  Then, we
can write \eqref{eq:linear_mixed} via an appropriate function,
$\ybf = \phibf_1(\Zbf_1,\Xibf_1)$.  
\iftoggle{conference}{The full paper \cite{pandit2020mlmatvamp-arxiv}
also discusses applications in sketched clustering
\cite{keriven2017sketching,byrne2019sketched}.}{
\subsection{Sketched Clustering}
A related problem arises in \emph{sketched clustering} \cite{keriven2017sketching}, where a massive dataset is nonlinearly compressed down to a short vector $\ybf\in\Real^n$, from which cluster centroids $\fbf_k\in\Real^p$, for $k=1,\dots,d$, are then extracted.
This problem can be approached via the optimization \cite{keriven2017compressive}
\begin{equation}
    \min_{\alphabf\geq\zero} \min_{\Fbf} \sum_{i=1}^{n} \bigg| y_i-\sum_{j=1}^d \alpha_j 
                                                     e^{\sqrt{-1} \xbf_i\tran\fbf_j} \bigg|^2 
    \label{eq:sketched-clustering},
\end{equation}
where $\xbf_i\in\Real^p$ are known i.i.d.\ Gaussian vectors.
An AMP approach to sketched clustering was developed in \cite{byrne2019sketched}.
For known $\alphabf$, the minimization corresponds to 
MAP estimation with the multi-layer matrix model with $\Zbf_0=\Fbf$, $\Wbf_1=\Xbf$
$\Zbf_1 = \Xbf\Fbf$ and using the output mapping,
\[
    \phibf_1(\Zbf_1,\Xibf) := \sum_{j=1}^d \alpha_j 
        e^{\sqrt{-1} \Zbf_{1,:j}} + \Xibf,
\]
where the exponential is applied elementwise and $\Xibf$ is i.i.d.\ Gaussian.  The mapping $\phi_1$
operates row-wise on $\Zbf_1$ and $\Xibf$.
}

\subsection{Learning the Input Layer of a Two-Layer Neural Network} 
\label{sec:2layerNN}
The matrix inference problem \eqref{eq:main_problem} can also be 
applied to 
learning the input layer weights in a two-layer neural network (NN).
Let $\Xbf \in \R^{N \times N_{\rm in}}$ 
and $\Ybf \in \R^{N \times N_{\rm out}}$ 
be training data corresponding to $N$ data samples.
Consider the two-layer NN model,
\begin{equation} \label{eq:two_layer}
    \Ybf = \sigma(\Xbf\bm{F}_1)\bm{F}_2  + \Xibf,
\end{equation}
with weight matrices $(\bm{F}_1,\bm{F}_2)$, 
componentwise activation function $\sigma(\cdot)$,
and noise $\Xibf$. \iftoggle{conference}{
The goal is to learn the weights of the first layer, $\bm{F}_1\in\Real^{N_{\rm in}\times N_{\rm hid}}$, from a 
dataset of $N$ samples in $(\Xbf,\Ybf)$ assuming the second 
layer weights $\Fbf_2$ are known.
}{
In \eqref{eq:two_layer}, the bias terms are omitted for simplicity. 
We used the notation ``$\bm{F}_\ell$'' for the weights,
instead of the standard notation ``$\bm{W}_\ell$,'' 
to avoid confusion when \eqref{eq:two_layer} is mapped to the multi-layer inference network \eqref{eq:main_problem}.
Now, our critical assumption is that the weights in the second layer, $\bm{F}_2$, are known.
The goal is to learn only the weights of the first layer, $\bm{F}_1\in\Real^{N_{\rm in}\times N_{\rm hid}}$, from a 
dataset of $N$ samples $(\Xbf,\Ybf)$. 
}

If the activation is ReLU, i.e., 
$\sigma(\bm{H}) = \max\{ \bm{H}, 0\}$ and $\Ybf$ has a single 
column (i.e.\ scalar output per sample),
and $\bm{F}_2$ has all positive entries, we can w.l.o.g.\ treat the weights $\bm{F}_2$ as fixed, since they can always be absorbed into the weights $\bm{F}_1$.  
In this case, $\ybf$ and $\Fbf_2$ are vectors and
we can write the $i$th entry of $\ybf$ as
\beq \label{eq:com_mac}
    y_i = \sum_{j=1}^d F_{2j}\sigma([\Xbf\bm{F}_1]_{ij}) + \xi_i 
        = \sum_{j=1}^d \sigma([\Xbf\bm{F}_1]_{ij}F_{2j}) + \xi_i 
\eeq
Thus, we can assume w.l.o.g.\ that $\Fbf_2$ is all ones.  
The parameterization \eqref{eq:com_mac} is sometimes referred to as the \emph{committee machine} \cite{tresp2000bayesian}.
The committee machine has been recently studied by AMP methods \cite{aubin2018committee} 
and mean-field methods \cite{mei2018mean} as a way to understand the dynamics of learning.  

To pose the two-layer learning problem as multi-layer inference, 
define
\[
    \bm{Z}_0 := \bm{F}_1, 
    \quad 
    \bm{W}_1 := \Xbf, 
    \quad
    \bm{Z}_1 := \Xbf\bm{F}_1 
    \quad
    \Xibf_2 := \Xibf ,
\]
then $\bm{Y}=\bm{Z}_2$, where $\bm{Z}_2$ is the output of a 2-layer inference network of the form in \eqref{eq:nntrue}:
\beq \label{eq:two_layer_ml}
    \Ybf = \bm{Z}_2 = \phibf_2(\bm{Z}_1, \Xibf_2) := \sigma(\bm{Z}_1)\bm{F}_2 + \Xibf_2.
\eeq
\iftoggle{conference}{}{
Note that $\bm{W}_1$ is known.  
Also, since we have assumed that $\bm{F}_2$ is known, the function $\phibf_2$ is known.
Finally, the function $\phibf_2$ is row-wise separable on both inputs.  
Thus, the problem of learning the input weights $\bm{F}_1$ is equivalent to learning the input $\bm{Z}_0$ of the network \eqref{eq:two_layer_ml}.
} 
\vspace{5pt}
\section{Multi-layer Matrix VAMP}

\subsection{MAP and MMSE inference}
Observe that the equations \eqref{eq:nntrue} define a Markov chain over these signals and thus the posterior $p(\Zbf|\Zbf_L)$ factorizes as
\begin{equation}\label{eq:posterior}
p(\Zbf|\Zbf_L)\propto p(\Zbf_0)\prod_{\ell=1}^{L-1} p(\Zbf_{\ell}|\Zbf_{\ell-1})\,p(\Ybf|\Zbf_{L-1}).
\end{equation}
where the transition probabilities $p(\Zbf_\ell|\Zbf_{\ell-1})$ are implicitly defined in equation \eqref{eq:nntrue} and depend on the statistics of noise terms $\Xibf_\ell$.
We consider both maximum \emph{a posteriori} (MAP) and minimum mean squared error (MMSE) estimation
for this posterior:
\begin{subequations}\label{eq:estimators}
\begin{align}
    \hat{\Zbf}_{\sf{map}} &= \argmax_{\Zbf}\  p(\Zbf|\Zbf_L)\label{eq:MAP_def}\\
    \hat{\Zbf}_{\msf{mmse}} & =\Exp[\Zbf|\Zbf_L]=\int \Zbf\, p(\Zbf|\Zbf_L)\dif\Zbf\label{eq:MMSE_def}.
\end{align}
\end{subequations}
\subsection{Algorithm Details}

The ML-Mat-VAMP for approximately computing the MAP and MMSE
estimates is similar to the ML-VAMP method in \cite{fletcher2018inference,pandit2019asymptotics}.
The specific iterations of ML-Mat-VAMP 
algorithm are shown in Algorithm \ref{algo:ml-mat-vamp}. 
The algorithm produces estimates by a sequence of forward and backward pass updates denoted by superscripts $^+$ and $^-$ respectively. 

The estimates $\Zbfhat^\pm_\ell$ are constructed by solving sequential problems $\Zbf=\{\Zbf_\ell\}_{\ell=0}^{\ell-1}$ into a sequence of smaller problems each involving estimation of a single activation or preactivation $\Zbf_\ell$ via \emph{estimation functions} $\{\Gbf_\ell^\pm(\cdot)\}_{\ell=1}^{L-1}$ which are selected
depending on whether one is
interested in MAP or MMSE estimation.  

To describe the estimation functions, we use the notation
that, 
for a positive definite matrix $\Gammabf$, define the inner product $\inner{\Abf, \Bbf}_{\Gamma}:=\Tr(\Abf\tran \Bbf \Gammabf)$ and let $\norm{\Abf}_{\Gammabf}$ denote the norm induced by this inner product. For $\ell = 1, \ldots, L-1$ define the approximate belief functions
\begin{equation}\label{eq:def_belief}
\begin{aligned}
    \MoveEqLeft b_\ell(\Zbf_\ell,\Zbf_{\ell-1}|\Rbf_\ell^-,\Rbf_{\ell-1}^+,\Gammabf_\ell^-,\Gammabf_{\ell-1}^+)
     \propto p(\Zbf_\ell|\Zbf_{\ell-1})
    \\
    &\times e^{-\frac{1}2\norm{\Zbf_\ell-\Rbf_\ell^-}_{\Gammabf_{\ell}^-}^2-\frac{1}{2}\norm{\Zbf_{\ell-1}-\Rbf_{\ell-1}^+}_{\Gammabf_{\ell-1}^+}^2}.
\end{aligned}
\end{equation}
Define $b_0(\Zbf_0|\Rbf_0^-,\Gammabf_0^-)$ and
$b_L(\Zbf_{L-1}|\Rbf_{L-1}^+,\Gammabf_{L-1}^+)$
similarly.

\begin{algorithm}[t]
\caption{Multilayer Matrix VAMP (ML-Mat-VAMP)}
\begin{algorithmic}[1]  \label{algo:ml-mat-vamp}
\REQUIRE{Estimators $\Gbf_0^+$, $\Gbf_L^-$, $\{\Gbf_\ell^\pm\}_{\ell=1}^{L-1}$.}\STATE{Set $\Rbf^-_{0\ell}=\zero$ and initialize $\{\Gammabf_{0\ell}^-\}_{\ell=0}^{L-1}$.}
\FOR{$k=0,1,\dots,N_{\rm it}-1$}

    \STATE{// \texttt{Forward Pass} }
    \STATE{$\Zbfhat^+_{k0} = \Gbf_0^+(\Rbf^-_{k0},\Gammabf_{k0}^-)$}
        \label{line:zp0}
    \STATE{$\Lambdabf^+_{k0} = \bkt{\tfrac{\partial \Gbf_0^+}{\partial \Rbf^-_{k,0}}(\Rbf^-_{k0},\Gammabf_{k0}^-)
            }^{-1}\Gammabf_{k,0}^-, $}
            \STATE{$ \Gammabf_{k,0}^+ = \Lambdabf_{k,0}^+-\Gammabf_{k,0}^-$}
            \label{line:alphap0}
    \STATE{$\Rbf^+_{k,0} = (\Zbfhat^+_{k,0}\Lambdabf_{k,0}^+ - \Rbf^-_{k,0}\Gammabf^-_{k,0})(\Gammabf_{k,0}^{+})^{-1}$}
            \label{line:rp0}
    \FOR{$\ell=1,\ldots,\Lm1$}
        \STATE{$\Zbfhat^+_{k\ell} =
        \Gbf_\ell^+(\Rbf^-_{k\ell},\Rbf^+_{k,\lm1},\Gammabf_{k\ell}^-, \Gammabf^+_{k,\ell-1})$}
        \label{line:zp}
        \STATE{$\Lambdabf^+_{k\ell} = \bkt{\tfrac{\partial \Gbf_\ell^+}{\partial \Rbf^-_{\ell}}(\Rbf^-_{k\ell},\Rbf^+_{k,\lm1},\Gammabf_{k\ell}^-, \Gammabf^+_{k,\ell-1})}^{-1}\Gammabf_{k,\ell}^-, $}
        \STATE{$ \Gammabf_{k,\ell}^+ = \Lambdabf_{k,\ell}^+-\Gammabf_{k,\ell}^-$}
            \label{line:alphap}
\STATE{$\Rbf^+_{k\ell} = (\Zbfhat^+_{k,\ell}\Lambdabf_{k,\ell}^+ - \Rbf^-_{k,\ell}\Gammabf^-_{k,\ell})(\Gammabf_{k,\ell}^{+})^{-1}$}     \label{line:rp}
    \ENDFOR
    \STATE{}

    \STATE{// \texttt{Backward Pass} }
    \STATE{$\Zbfhat^-_{k,\Lm1} =
        \Gbf_{L}^-(\Rbf^+_{k,\Lm1},\Gammabf_{k,\Lm1}^+)$}   \label{line:znL}
    \STATE{$\Lambdabf^-_{k,\Lm1} = \bkt{\tfrac{\partial
        \Gbf_{L}^-}{\partial \Rbf^+_{k,\Lm1}}(\Rbf^+_{k,\Lm1},\Gammabf_{k,\Lm1}^+) }^{-1}\Gammabf_{k,\Lm1}^+, $}
    \STATE{$ \Gammabf_{k,\Lm1}^- = \Lambdabf_{k,\Lm1}^--\Gammabf_{k,\Lm1}^+$}
        \label{line:alphanL}
   \STATE{$\Rbf^-_{\kp1,\Lm1} = (\Zbfhat^-_{k,\Lm1}\Lambdabf_{k,\Lm1}^- - \Rbf^+_{k,0}\Gammabf^+_{k,0})(\Gammabf_{k,0}^{-})^{-1}$}
            \label{line:rnL}
    \FOR{$\ell=\Lm1,\ldots,1$}
        \STATE{$\Zbfhat^-_{k+1,\ell-1} =
            \Gbf_{\ell}^-(\Rbf^-_{k+1,\ell},\Rbf^+_{k,\ell-1},\Gammabf_{k+1,\ell}^-, \Gammabf^+_{k,\ell-1})$}
            \label{line:zn}
        \STATE{$\Lambdabf^-_{k+1,\ell-1} = \bkt{\tfrac{\partial
        \Gbf_{\ell}^-}{\partial \Rbf^+_{\ell-1}}(\cdots)}^{-1}\Gammabf_{k,\lm1}^+, $}
        \STATE{$ \Gammabf_{k+1,\ell}^- = \Lambdabf_{k,\ell}^--\Gammabf_{k,\ell}^+$}
            \label{line:alphan}
        \STATE{$\Rbf^-_{\kp1,\ell-1} = (\Zbfhat^-_{k,\ell}\Lambdabf_{k,\ell}^- - \Rbf^+_{k,\ell}\Gammabf^+_{k,\ell})(\Gammabf_{k+1,\ell}^{-})^{-1}$}
            \label{line:rn}
    \ENDFOR

\ENDFOR
\end{algorithmic}
\end{algorithm}

The MAP and MMSE estimation functions are then given by the
MAP and MMSE estimates for these belief densities,
\begin{subequations}\label{eq:subestimators}
\begin{align}
\Gbf_{\ell,\msf{map}}^\pm= (\Zbfhat_\ell^+,\Zbfhat_{\ell-1}^-)_{\msf{map}}&={\rm argmax}  \ b_\ell(\Zbf_\ell,\Zbf_{\ell-1})
 \label{eq:denoiser_map}\\
 \Gbf_{\ell,\msf{mmse}}^\pm=(\Zbfhat_\ell^+,\Zbfhat_{\ell-1}^-)_{\msf{mmse}}&= \Exp[(\Zbf_\ell,\Zbf_{\ell-1})|b_\ell]  
 \label{eq:denoiser_mmse}
\end{align}
\end{subequations}
where the expectation is with respect to the normalized density proportional to $b_\ell$. 
Thus, the ML-Mat-VAMP algorithm reduces the joint estimation
of the vectors $(\Zbf_0,\ldots,\Zbf_{L-1})$ 
to a sequence of simpler estimations on sub-problems with terms $(\Zbf_{\ell-1},\Zbf_\ell)$.
We refer to these subproblems as denoisers and denote their solutions by $\Gbf_\ell^\pm$, so that $\Zbfhat^+_\ell=\Gbf_\ell^+$ and $\Zbfhat^-_{\ell-1}=\Gbf_\ell^-$ corresponding to lines \ref{line:zp} and \ref{line:zn} of Algorithm~\ref{algo:ml-mat-vamp}. The denoisers $\Gbf_0^+$ and $\Gbf_L^-$, which provide updates to $\Zbfhat_0^+$ and $\Zbfhat_{L-1}^-$, are defined in a similar manner via $b_0$ and $b_L$ respectively.

\iftoggle{conference}{As shown in detail in the full paper \cite{pandit2020mlmatvamp-arxiv}, the}{The} estimation 
functions \eqref{eq:subestimators} can be easily computed
for the multi-layer matrix network. An important characteristic of these estimators is that they can be computed using maps which are row-wise separable over their inputs and hence are easily parallelizable.
\iftoggle{conference}{}{To simplify notation, we denote the precision parameters for denoisers $\Gbf_\ell^\pm$ in the $k^{\rm th}$ iteration by \begin{equation}\label{eq:thetadef}
\begin{aligned}
\MoveEqLeft\Thetabf_{k\ell}^+:=(\Gammabf_{k\ell}^-,\Gammabf_{k,\ell-1}^+),
\quad  &\Thetabf_{k\ell}^-&:= (\Gammabf_{k+1,\ell}^-,\Gammabf_{k,\ell-1}^+),\\
\MoveEqLeft\Thetabf_{k0}^+:= \Gammabf_{k0}^-, 
\quad &\Thetabf_{kL}^-&:=\Gammabf_{k,L-1}^+.
\end{aligned}
\end{equation}
}
\subsubsection{Non-linear layers}: For $\ell$ even, since the rows of $\Xibf_\ell$ are i.i.d., the belief density density $b_\ell(\Zbf_\ell,\Zbf_{\lm1}|\cdot)$ from \eqref{eq:def_belief} factors as a product across rows.
\begin{align*}
    b_\ell(\Zbf_\ell,\Zbf_{\lm1})= \prod_n b_\ell([\Zbf_{\ell}]_{n:},[\Zbf_{\lm1}]_{n:})
\end{align*}
Thus, the MAP and MMSE estimates \eqref{eq:subestimators} can be     performed over $d$-dimensional variables where     $d$ is the number of entries in each row.  There
    is no joint estimation across the different $n_\ell$ rows.
\subsubsection{Linear layers}  When $\ell$ is odd,
    the density $b_\ell(\Zbf_\ell,\Zbf_{\ell-1}|\cdot)$
    in \eqref{eq:def_belief} is a Gaussian.  Hence,
    the MAP and MMSE estimates agree and can be
    computed via least squares.  

\iftoggle{conference}{}{

Although for linear layers $[\Gbf_\ell^+,\Gbf_\ell^-](\Rbf_\ell^-,\Rbf_{\lm1}^+,\Thetabf_\ell)$ is not row-wise separable over $(\Rbf_\ell^-,\Rbf_{\lm1})$, it can be computed using another row-wise denoiser 
$[\wt\Gbf_\ell^+,\wt\Gbf_\ell^-]$ via the SVD of the weight matrix $\Wbf_\ell=\Vbf_\ell\diag(\Sbf_\ell)\Vbf_{\lm1}$ as follows. Note that the SVD is only needed to be performed once.:
\begin{align*}
    [\Gbf_\ell^+,\Gbf_\ell^-] &= \underset{\Zbf_\ell,\Zbf_{\lm1}}{\rm argmax} \norm{\Zbf_\ell-\Wbf_\ell\Zbf_{\lm1}-\Bbf_\ell}_{\Nbf_\ell}^2\\
    &+\norm{\Zbf_\ell-\Rbf_\ell^-}_{\Gammabf_{\ell}^-}^2+\norm{\Zbf_{\lm1}-\Rbf_{\lm1}^+}_{\Gammabf_{\lm1}^+}^2\\
    \overset{\rm (a)}= \underset{\Zbf_\ell,\Zbf_{\lm1}}{\rm argmax}& \norm{\Vbf_\ell\tran\Zbf_\ell-\diag(\Sbf_\ell)\Vbf_{\lm1}\Zbf_{\lm1}-\Vbf_\ell\tran\Bbf_\ell}_{\Nbf_\ell}^2\\
    &+\norm{\Vbf_\ell\tran\Zbf_\ell-\Vbf_\ell\tran\Rbf_\ell^-}_{\Gammabf_{\ell}^-}^2\\
    &+\norm{\Vbf_{\lm1}\Zbf_{\lm1}-\Vbf_{\lm1}\Rbf_{\lm1}^+}_{\Gammabf_{\lm1}^+}^2\\
\overset{\rm (b)}= [\Vbf_\ell\tran\wt\Gbf_\ell^+,&\Vbf_{\lm1}\wt\Gbf_\ell^-](\V_\ell\tran\Rbf_\ell,\V_{\lm1}\Rbf_{\lm1},\Thetabf_\ell)
\end{align*}
where (a) follows from the rotational invariance of the norm, and (b) follows from the definition of
denoiser $[\wt\Gbf_\ell^+,\wt\Gbf_\ell^-](\wt\Rbf_\ell^-,\wt\Rbf_{\lm1}^+,\Thetabf_\ell)$ given below
\begin{align}\label{eq:Gtildedef}
[\wt\Gbf_\ell^+,\wt\Gbf_\ell^-] &:= \underset{\wt\Zbf_\ell,\wt\Zbf_{\lm1}}{\rm argmax} \norm{\wt\Zbf_\ell-\diag(\Sbf_\ell)\wt\Zbf_{\lm1}-\wt\Bbf_\ell}_{\Nbf_\ell}^2\nonumber\\
    &+\norm{\wt\Zbf_\ell-\wt\Rbf_\ell^-}_{\Gammabf_{\ell}^-}^2+\norm{\wt\Zbf_{\lm1}-\wt\Rbf_{\lm1}^+}_{\Gammabf_{\lm1}^+}^2
\end{align}
Note that the optimization problem in \eqref{eq:Gtildedef}, is decomposable accross the rows of variables $\wt\Zbf_\ell$ and $\wt\Zbf_{\lm1}$, and hence $[\wt\Gbf_\ell^+,\wt\Gbf_\ell^-]$ operates row-wise on its inputs.
}

\section{Analysis in the Large System Limit }

\iftoggle{conference}{}{
\subsection{Large System Limit}
}

We follow the analysis framework of the ML-VAMP
work \cite{fletcher2018inference,pandit2019asymptotics},
which is itself based on the original AMP analysis 
in \cite{BayatiM:11}.
This analysis is based on considering the asymptotics of
cerain large random problem instances.
Specifically, we consider a sequence of problems \eqref{eq:nntrue}
indexed by $N$ such that for each problem the dimensions $n_\ell(N)$ are growing so that $ \lim_{N\rightarrow\infty} \tfrac{n_\ell}{N}= \beta_\ell \in (0,\infty)$ are scalar constants. Note that $d$ is finite and does not grow with $N$.

\subsubsection{Distributions of weight matrices}
For $\ell=1,3,\ldots,L-1$, 
we assume that the weight matrices $\Wbf_\ell$ are generated
via the singular value decomposition,
\begin{align}
    \label{eq:svd}
\Wbf_\ell = \Vbf_\ell\diag(\Sbf_\ell)\Vbf_{\ell-1}
\end{align}
where $\V_\ell\in\Real^{n_\ell\times n_\ell}$ are Haar distributed over orthonormal matrices and 
$\Sbf_\ell=(s_{\ell,1},\ldots,s_{\ell,{\rm min}\{n_\ell,n_{\lm1}\}})$. 
We will describe the distribution of the 
components $\Sbf_\ell$ momentarily.

\subsubsection{Assumption on Denoisers}
We assume that the non-linear denoisers $\Gbf_{2k}^\pm$ act row-wise on their inputs $(\Rbf_{2k}^-,\Rbf_{2k-1}^+)$. Further these operators and their Jacobian matrices $\tfrac{\partial\Gbf_{2k}^+}{\partial\Rbf_{2k}^-},\tfrac{\partial\Gbf_{2k}^-}{\partial\Rbf_{2k-1}^+},\tfrac{\partial\Gbf_{0}^+}{\partial\Rbf_{0}^-}, \tfrac{\partial\Gbf_{L}^-}{\partial\Rbf_{L-1}^+}$ are  \textit{uniformly Lipschitz continuous}, the definition of which is provided in \iftoggle{conference}{the full paper \cite{pandit2020mlmatvamp-arxiv}.}{\ref{sec:lsl_details}.}

\subsubsection{Assumption on Initialization, True variables}
The distribution of the remaining variables are described by a weak limit:
For a matrix sequence $\bm{X}:=\bm{X}(N)\in\Real^{N\times d}$, by the notation $\bm{X}\xRightarrow{2} X$ we mean that  there exists a random variable $X$ in $\Real^d$ with $\Exp\|X\|^2 < \infty$ such that $\lim_{N\rightarrow\infty}\sum_{i=1}^N\psi(\bm{X}_{i:})=\Exp\,\psi(X)$ almost surely, for any  bounded continuous function $\psi:\Real^d\rightarrow\Real$, as well as for quadratic functions $\x\T\bm{P}\x$ for any $\bm{P}\in \Real^{d\times d}_{\succeq 0}$. This is also referred to as Wasserstein-2 convergence \cite{montanari2019generalization}. For e.g., this property is satisfied for a random $\X$ with i.i.d.\ rows with bounded second moments, but is more general, since it applies to deterministic matrix sequences as well. More details on this weak limit are given in the 
\ref{sec:lsl_details}.

Let $\wb\Bbf_\ell:=\Vbf_\ell\tran\Bbf_\ell$, and $\wb\Sbf_\ell\in\Real^{n_\ell}$ be the zero-padded vector of singular values of $\Wbf_\ell$, and let $\taubf_{0\ell}^-\in\Real^{d\times d}_{\succ 0}$. Then we assume that the following empirical convergences hold.\begin{subequations}
\begin{align*}
    &\Zbf^0_{0} \xRightarrow{2} Z^0,\\
    &(\Xibf_{\ell},\Rbf_{0\ell}^--\Zbf^0_{\ell})
    \xRightarrow{2}(\Xi_{\ell},Q_{0\ell}^-), \quad &\ell
    \mbox{ even}  \\
    &(\wb\Sbf_{\ell},\wb\Bbf_{\ell},\Xibf_{\ell},\Vbf_\ell\T(\Rbf_{0\ell}^--\Zbf^0_{\ell}))  \xRightarrow{2}(S_{\ell},\wb B_{\ell},\Xi_{\ell},Q_{0\ell}^-),
    \quad &\ell \mbox{ odd}
\end{align*}
\end{subequations}
where $S_\ell\in\Real_{\geq 0}$ is bounded,  $\wb B_\ell\in\Real^d$ is bounded, $\Xi_{2\ell-1}\sim\mc N(\zero,\Nbf_{2\ell-1}^{-1})$, and 
\begin{align}\label{eq:init_random_vars}
Q_{0\ell}^-\sim \mc N(\zero,\taubf_{0\ell}^-),\quad \ell=0,1,\ldots,L-1
\end{align}
are all pairwise independent random variables. We also assume that the sequence of initial matrices $\{\Gammabf_{0\ell}^-\}$ satisfy the following convergence pointwise
\begin{align}\label{eq:init_precisions}
\Gammabf_{0\ell}^-(N)\rightarrow \Gammabfbar_{0\ell}^-,\quad \ell=0,1,\ldots,L-1
\end{align}

\iftoggle{conference}{}{
\subsection{Main Result}

The main result of this paper concerns the empirical distribution of the rows $[\Zbfhat_\ell^\pm]_{n:},[\Rbf_\ell^\pm]_{n:}$ of the iterates of Algorithm \ref{algo:ml-mat-vamp}. It characterizes the asymptotic behaviour of these empirical distributions in terms of $d$-dimensional random vectors which are either Gaussians or functions of Gaussians. Let $G_\ell^\pm$ denote maps $\Real^{1\times d}\rightarrow\Real^{1\times d},$ such that \eqref{eq:subestimators}, \ie, 
\begin{align*}[\Gbf_\ell^\pm(\Rbf_\ell^-,\Rbf_{\lm1}^+,\Thetabf)]_{n:} = G^\pm_\ell([\Rbf_\ell^-]_{n:},[\Rbf_{\lm1}^+]_{n:},\Thetabf).
\end{align*}
Having stated the requisite definitions and assumptions, we can now state our main result.
}
\setlength{\belowdisplayskip}{7pt}
\begin{theorem} \iftoggle{conference}{\label{thm:se_short}}{\label{thm:main_result}}
For a fixed iteration index $k\geq 0$, 
there exist deterministic matrices $\Kbf_{k\ell}^+\in\Real^{2d\times 2d}_{\succ 0}$, and $\taubf_{k\ell}^-,\wb\Gammabf^+_{k\ell}$ and $\wb\Gammabf^-_{k\ell},\in\Real^{d\times d}_{\succ 0}$ such that for even $\ell$:
\begin{align}
\left(\begin{matrix}\Zbf_{\ell-1}^0\\\Zbf_{\ell}^0\\\Zbfhat_{k,\ell-1}^-\\\Zbfhat_{k\ell}^+\end{matrix}\right) \xRightarrow{2}
    \begin{pmatrix}\mathsf{A}\\
    \wt{\mathsf{A}}\\
    G_\ell^-(\mathsf{C}+\wt{\mathsf{A}},\mathsf{B}+{\mathsf{A}},\wb\Gammabf_{k\ell}^-,\wb\Gammabf_{k,\ell-1}^+)\\
    G_\ell^+(\mathsf{C}+\wt{\mathsf{A}},\mathsf{B}+{\mathsf{A}},\wb\Gammabf_{k\ell}^-,\wb\Gammabf_{k,\ell-1}^+)
    \end{pmatrix}
    \label{eq:zconv_even}
\end{align}
where $(\mathsf{A},\mathsf{B})\sim\mc N(\zero,\Kbf_{k,\ell-1}^+)$, $\mathsf{C}\sim \mc N(\zero,\taubf_{k\ell}^-)$, 
$\wt{\mathsf{A}} = \phi_\ell(\mathsf{A},\Xi_\ell)$ and $(\msf A,\msf B),\msf C$ are independent. \iftoggle{conference}{}{For $\ell=0$, the same result holds where the $1^{\rm st}$ and $3^{\rm rd}$ terms are dropped, whereas for $\ell=L,$ the $2^{\rm nd}$ and $4^{\rm th}$ terms are dropped.}
\iftoggle{conference}{}{

Similarly, for odd $\ell$:
\begin{align}
    \left(\begin{matrix}\Vbf_{\ell-1}\tran\Zbf_{\ell-1}^0\\\Vbf_{\ell-1}\tran\Zbf_{\ell}^0\\\Vbf_{\ell}\Zbfhat_{k,\ell-1}^-\\\Vbf_{\ell}\Zbfhat_{k\ell}^+\end{matrix}\right) \xRightarrow{2}
    \begin{pmatrix}\mathsf{A}\\
    \wt{\mathsf{A}}\\
    G_\ell^-(\mathsf{C}+\wt{\mathsf{A}},\mathsf{B}+{\mathsf{A}},\wb\Gammabf_{k\ell}^-,\wb\Gammabf_{k,\ell-1}^+)\\
    G_\ell^+(\mathsf{C}+\wt{\mathsf{A}},\mathsf{B}+{\mathsf{A}},\wb\Gammabf_{k\ell}^-,\wb\Gammabf_{k,\ell-1}^+)
    \end{pmatrix}
    \label{eq:zconv_odd}
\end{align}
where $(\mathsf{A},\mathsf{B})\sim\mc N(\zero,\Kbf_{k,\ell-1}^+)$, $\mathsf{C}\sim \mc N(\zero,\taubf_{k\ell}^-)$, 
$\wt{\mathsf{A}} = S_\ell\, \mathsf{A} + \wb{B}_{\ell}+\Xi_\ell$ and $(\msf A,\msf B),\msf C$ are independent. 
}

Furthermore for $\ell=0,1,\ldots L-1$, we have
\begin{align}
(\Gammabf_{k\ell}^\pm,\Lambdabf_{k\ell}^\pm)\xrightarrow{a.s.} (\Gammabfbar_{k\ell}^\pm,\Lambdabfbar_{k\ell}^\pm).
\end{align}
\end{theorem}

\iftoggle{conference}{Due to space limitations, 
the proof of the result is given in the full paper
\cite{pandit2020mlmatvamp-arxiv} along with a similar statement for the indices when $\ell$ is odd.  In addition, the}
{The} parameters in the distribution, $\{\Kbf_{k\ell}^+,\taubf_{k\ell}^-,\wb\Gammabf_{k\ell}^\pm,\wb\Lambdabf_{k\ell}^\pm\}$ are deterministic and
can be computed via a set of recursive equations
called the \emph{state evolution} or SE. 
\iftoggle{conference}{The SE equations are also provided in the full paper.}{The SE equations are provided in \ref{sec:state_evolution}}
The result is similar to those for ML-VAMP in
\cite{fletcher2018inference,pandit2019inference}
except that the SE equations for ML-Mat-VAMP 
involve $d \times d$ and $2d\times 2d$ 
matrix terms; the ML-VAMP SE requires
only require scalar and $2 \times 2$ terms.
The result holds for both MAP inference and MMSE inference, the only difference is implicit, \ie, the choice of denoiser 
$\Gbf_\ell(\cdot)$ from eqn. \eqref{eq:subestimators}.

The importance of Theorem~\iftoggle{conference}{\ref{thm:se_short}}{\ref{thm:main_result}}
is that the rows of the iterates of the ML-Mat-VAMP
Algorithm ($\Zbfhat_{k,\ell-1}^-,\Zbfhat_{k\ell}^+$ in Algorithm~\ref{algo:ml-mat-vamp}) and the rows of the
corresponding true
values, $\Zbf_{\ell-1}^0,\Zbf_{\ell}^0$, have a simple, asymptotic  random vector description of a typical row.  
We will call this the ``row-wise" model.  \iftoggle{conference}{In this model,}{According to this model, for even $\ell$,} 
the rows of 
$\Zbf_{\ell-1}^0$ converge to a Gaussian $\mathsf{A} \in \R^d$ 
and the rows of $\Zbf_{\ell}^0$ converge to the 
output of the Gaussian through the row-wise function $\phi_\ell$,
$\wt{\mathsf{A}} = \phi_\ell(\mathsf{A},\Xi_\ell)$.
Then the
rows of the estimates $\Zbfhat_{k,\ell-1}^-,\Zbfhat_{k\ell}^+$
asymptotically approach to the outputs of row-wise
estimation function $G^+(\cdot)$ and $G^+(\cdot)$
supplied by $\mathsf{A}$ and $\wt{\mathsf{A}}$ corrupted
with Gaussian noise. \iftoggle{conference}{}{A similar convergence holds for odd $\ell$.}

This ``row-wise" model
enables exact an analysis of the performance of the
estimates at each iteration. For example,
to compute a weighted mean squared error (MSE)  metric at iteartion $k$, the convergence \eqref{eq:zconv_even} shows
that,
\begin{align*}
    \MoveEqLeft \tfrac{1}{n_\ell}\norm{\Zbfhat_{k\ell}^+-\Zbf^0_\ell}^2_{\Hbf} \xrightarrow{a.s.} \Exp\|\Gbf_\ell^+(\msf C+\wt{\msf A},\msf B+\msf A,\Thetabf_{k\ell})-\wt{\msf{A}}\|^2_{\Hbf},
\end{align*}
for even $\ell$ and any positive semi-definite matrix 
$\Hbf \in \R^{d\times d}$.  The norm on the left-hand
above acts row-wise,
$\|\Zbf\|^2_{\Hbf} := \sum_i \|\Zbf_{i:}\|^2_{\Hbf}$.
Hence, this asymptotic MSE can be evaluated via expectations
of $d$-dimensional variables from the SE.
Similarly, one can obtain exact answers for any other row-wise performance
metric of $\{(\Zbfhat_{k\ell}^\pm,\Zbf_\ell^0)\}_{\ell}$ for any $k$.

\section{Numerical Experiments}

\begin{figure}
    \centering
    \includegraphics[width=.45\textwidth]{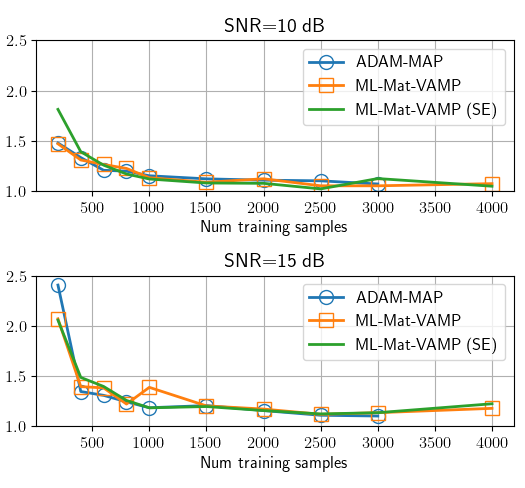}
    \caption{Test error in learning the first layer of a 2 layer neural network using ADAM-based gradient descent, ML-Mat-VAMP and its 
    state evolution prediction. }
    \label{fig:mse_ntr}
\end{figure}
We consider the problem of learning the input layer of a two layer neural network as described in Section \ref{sec:2layerNN}. We learn the weights  $\bm{F}_1$ of the first layer of a two-layer network by solving problem \eqref{eq:two_layer_ml}.  The LSL analysis in this case
corresponds to the input size $n_{\rm in}$ and number of samples
$N$ going to infinity with the number of hidden units being fixed.
Our experiment take  $d=4$ hidden units,
$N_{\rm in}=100$ input units, $N_{\rm out}=1$ output unit, sigmoid activations and variable number of samples $N$.
The weight vectors $\bm{F}_1$ and $\bm{F}_2$ are generated as i.i.d.\ Gaussians with zero mean
and unit variance.  The input $\Xbf$ is also i.i.d.\ Gaussians with variance $1/N_{\rm in}$
so that the average pre-activation has unit variance.
Output noise is added at two levels of 10 and 15~dB relative to the mean.
We generate 1000 test samples and a variable number of training samples that ranges from 
200 to 4000.  For each trial and number of training samples, we compare three methods:
(i) MAP estimation where the MAP loss function is minimized by the ADAM optimizer 
\cite{kingma2014adam} in the Keras package of Tensorflow;
(ii) the M-VAMP algorithm run for 20 iterations and (iii) the M-VAMP state evolution
prediction.  The MAP-ADAM algorithm is run for 100 epochs with a learning rate $= 0.01$.
The expectations in the M-VAMP SE are estimated via Monte-Carlo sampling (hence
there is some variation).  

Given an estimate $\wh{\Fbf}_1$ and
true value $\Fbf_1^0$, we can compute the test error as follows:
Given a new sample $\xbf$, the true
and predicted pre-activations will be $\zbf_1 = (\Fbf_1^0)\tran\xbf$
and $\zbfhat_1 = \wh{\Fbf}_1\tran\xbf$.  Thus, if the new sample $\xbf \sim {\mathcal N}(0,\tfrac{1}{N_{\rm in}}\Ibf)$, the true and predicted 
pre-activations, $(\zbf_1,\zbfhat_1)$,
will be jointly Gaussian with covaraince equal to the 
empirical $2d\times 2d$ 
covariance matrix of the rows of $\Fbf_1^0$ and $\wh{\Fbf}_1$:
\beq \label{eq:Ktest}
    \Kbf := \tfrac{1}{N_{\rm in}} \sum_{k=1}^{N_{\rm in}}
    \ubf_k\tran \ubf_k, \quad 
    \ubf_k = \left[\Fbf_{1,k:}~ \wh{\Fbf}_{1,k:}\right]
\eeq
From this covariance matrix,
we can estimate the test error,
\[
    \Exp|y-\wh{y}|^2 = \Exp 
    |\Fbf_2\tran(\sigma(\zbf_1)-\sigma(\zbfhat_1) |^2,
\]
where the expectation is taken over the Gaussian $(\zbf_1,\zbfhat_1)$
with covariance $\Kbf$.  
Also, since \eqref{eq:Ktest} is a row-wise operation, it  
can be predicted from the ML-Mat-VAMP SE.  Thus, the SE can also 
predict the asymptotic test error.
The normalized test error for ADAM-MAP, ML-Mat-VAMP and the ML-Mat-VAMP SE
are plotted in Fig.~\ref{fig:mse_ntr}.
The normalized test error is defined as the ratio of the MSE on the test samples to the optimal MSE.
Hence, a normalized MSE of one is the minimum value.

Note that since ADAM and ML-Mat-VAMP are solving the same optimization problem, they perform similarly as expected. The main message of this paper is not to develop an algorithm that outperforms ADAM, but rather an algorithm that has theoretical guarantees. The key property of ML-Mat-VAMP is that its asymptotic behavior at all the iterations can be exactly predicted by the state evolution equations. In this example, 
Fig.~\ref{fig:mse_ntr} shows tat the normalized test MSE predicted via state evolution (plotted in green) matches the normalized MSE of ML-Mat-VAMP estimates (plotted in orange). 
\section*{Conclusions}

We have developed a general framework for analyzing 
inference in multi-layer
networks with matrix valued quantities in certain high-dimensional 
random settings.  For learning the input layer of a two layer
network, the methods enables precise predictions of the expected
test error as a function of the parameter statistics, 
numbers of samples and noise level.  This analysis can thus in turn
be valuable in understanding key properties such as generalization error.
Future work will look to extend these to more complex networks.

\paragraph*{Acknowledgements}
The work of P. Schniter was supported by NSF grant
1716388. The work of P.\ Pandit, M.\ Saharee-Ardakan and A.\ K.\ Fletcher was supported in part by the NSF Grants 1738285 and 1738286, ONR Grant N00014-15-1-2677. The work of
S. Rangan was supported in part by NSF grants
1116589, 1302336, and 1547332, NIST, SRC and the the industrial affiliates of NYU Wireless.

\iftoggle{conference}{}{
\clearpage
\setcounter{section}{0}
\setcounter{subsection}{0}
\renewcommand\thesection{Appendix \Alph{section}}
\renewcommand\thesubsection{\thesection.\arabic{subsection}}

\section{State Evolution Equations}
\label{sec:state_evolution}

\begin{algorithm}
\setstretch{1.1}
\caption{State Evolution for ML-Mat-VAMP (Algo. \ref{algo:ml-mat-vamp})}
\begin{algorithmic}[1]  \label{algo:ml-mat-vamp_se}

\REQUIRE{
Functions $\{f^0_\ell\}$ from \eqref{eq:f0def}, $\{h_\ell^\pm\}$ from \eqref{eq:hdef}, and $\{f^\pm_{\ell}\}$ from \eqref{eq:fdef}. Perturbation random variables $\{W_\ell\}$ from \eqref{eq:Wdef}. Initial random vectors $\{Q_{0\ell}^-\}_{\ell=0}^{L-1}$ with Initial covariance matrices $\{\taubf_{0\ell}^-\}_{\ell=0}^{L-1}$ from \eqref{eq:init_random_vars}. Initial matrices $\{\Gammabfbar_{0\ell}^-\}_{\ell=0}^L$ from \eqref{eq:init_precisions}.
}
\vspace{10pt}
\STATE{// \texttt{Initial Pass}}
    \label{line:qinit_se_mlvamp}
\STATE{$Q^0_0 = W_0$, $\taubf^0_0 = \Cov(Q^0_0)$ and $P^0_0 \sim \Norm(\zero,\taubf^0_0)$} \label{line:p0init_se_mlvamp}
\FOR{$\ell=1,\ldots,\Lm1$}
    \STATE{$Q^0_\ell=f^0_\ell(P^0_{\lm1},W_\ell)$}
    \label{line:q0_init}
    \STATE{$P^0_\ell \sim \Norm(\zero,\taubf^0_\ell)$,\qquad
            $\taubf^0_\ell = \Cov(Q^0_\ell)$} \label{line:pinit_se_mlvamp}
\ENDFOR
\vspace{10pt}
\FOR{$k=0,1,\ldots$}
    \STATE{// \texttt{Forward Pass} }
\STATE{$\wh{Q}^+_{k0} = h^+_0(Q_{k0}^-,W_0,\Thetabfbar^+_{k0})$} \label{line:qhat0_se_mlvamp}
    \STATE{$\Lambdabfbar_{k0}^+ = (\Exp\tfrac{\partial \wh{Q}^+_{k0}}{\partial Q_{0}^-})^{-1}\Gammabfbar_{k,0}^-$}
    \STATE{$\Gammabfbar_{k0}^+ = \Lambdabfbar_{k0}^+-\Gammabfbar_{k0}^-
$}
    \STATE{$Q_{k0}^+ = f^+_{0}(Q_{k0}^-,W_0,\Omegabfbar^+_{k0})$}  \label{line:q0_se_mlvamp}
    \STATE{$(P^0_0,P_{k0}^+) \sim \Norm(\zero,\Kbf_{k0}^+)$,
        $\qquad\Kbf_{k0}^+ := \Cov(Q^0_0,Q_{k0}^+)$}

        \label{line:p0_se_mlvamp}
        \vspace{10pt}
    \FOR{$\ell=1,\ldots,L-1$}
        \STATE{$\wh{Q}^+_{k\ell} = h^+_\ell(P^0_{\lm1},P^+_{k,\lm1},Q_{k\ell}^-,W_\ell,\Thetabfbar^+_{k\ell})$} \label{line:qhat_se_mlvamp}
        \STATE{$\Lambdabfbar_{k\ell}^+ = (\Exp\tfrac{\partial \wh{Q}^+_{k\ell}}{\partial Q_{k\ell}^-})^{-1}\Gammabfbar_{k\ell}^-$}
        \STATE{$\Gammabfbar_{k\ell}^+ = \Lambdabfbar_{k\ell}^+-\Gammabfbar_{k\ell}^-
$}
            \label{line:lamp_se_mlvamp}
        \STATE{$Q_{k\ell}^+ = f^+_{\ell}(P^0_{\lm1},P^+_{k,\lm1},Q_{k\ell}^-,W_\ell,\Omegabfbar^+_{k\ell})$}
            \label{line:qp_se_mlvamp}
        \STATE{$(P^0_\ell,P_{k\ell}^+) \sim \Norm(\zero,\Kbf_{k\ell}^+)$,
            $\quad\Kbf_{k\ell}^+ := \Cov(Q^0_\ell,Q_{k\ell}^+)$} \label{line:pp_se_mlvamp}
\ENDFOR
\vspace{10pt}
    \STATE{// \texttt{Backward Pass} }
    \STATE{$\wh{P}_{\kp1,\Lm1}^- = h^-_{L}(P^0_{\Lm1},P_{k,\Lm1}^+,{W_L},\Thetabfbar^-_{\kp1,L})$}
    \label{line:phatL_se_mlvamp}
    \STATE{$\Lambdabfbar_{k+1,L}^- = (\Exp\tfrac{\partial \wh{P}_{\kp1,\Lm1}^-}{\partial P_{\Lm1}^+})^{-1}\Gammabfbar_{kL}^+$}
    \STATE{$\Gammabfbar_{k+1,\Lm1}^- = \Lambdabfbar_{k+1,\Lm1}^--\Gammabfbar_{k,\Lm1}^+,$}
\STATE{$P_{\kp1,\Lm1}^- = f^-_{L}(P^0_{\Lm1},P_{k,\Lm1}^+,W_L,\Omegabfbar^-_{\kp1,L})$}  \label{line:pL_se_mlvamp}
\STATE{$Q_{\kp1,\Lm1}^- \sim \Norm(\zero,\taubf_{\kp1,\Lm1}^-),$$\ \ \taubf_{\kp1,\Lm1}^- := \Cov(P^-_{\kp1,\Lm1})$}
    \label{line:qL_se_mlvamp}
    \FOR{$\ell=\Lm2,\ldots,0$}
    \STATE{$\wh{P}_{\kp1,\ell}^- = h^-_{\ell}(P^0_{\ell},P_{k\ell}^+,Q^-_{k+1,\ell+1},W_\ell,\Thetabfbar^-_{\kp1,\ell})$}
    \label{line:phatn_se_mlvamp}
    \STATE{$\Lambdabfbar_{k+1,\ell}^- = (\Exp\tfrac{\partial \wh{P}_{\kp1,\ell}^-}{\partial P_{k,\ell}^+})^{-1}\Gammabfbar_{k,\ell}^+$}
    \STATE{$\Gammabfbar_{k+1,\ell}^- = \Lambdabfbar_{k+1,\ell}^--\Gammabfbar_{k,\ell}^+,$}
\STATE{$P_{\kp1,\ell}^- =
    f^-_{\ell}(P^0_{\ell},P^+_{k\ell},Q_{\kp1,\ell+1}^-,W_\ell,\Omegabfbar^-_{k+1,\ell})$}
    \label{line:pn_se_mlvamp}
\STATE{$Q_{\kp1,\ell}^- \sim \Norm(\zero,\taubf_{\kp1,\ell}^-),$ $\quad \taubf_{\kp1,\ell}^- := \Cov(P_{\kp1,\ell}^-)$}  \label{line:qn_se_mlvamp}
    \ENDFOR
\ENDFOR
\end{algorithmic}
\end{algorithm} 
The state evolution equations given in Algo. \ref{algo:ml-mat-vamp_se} define an iteration indexed by $k$ of constant matrices $\{\Kbf_{k\ell}^+,\taubf_{kl}^-,\Gammabfbar_{kl}^\pm\}_{\ell=0}^L.$ These constants appear in the statement of the main result in Theorem \ref{thm:main_result}. The iterations in Algo. \ref{algo:ml-mat-vamp_se} also iteratively define a few $\Real^{1\times d}$ valued random vectors $\{Q_\ell^0,P_{\ell}^0,Q_{k\ell}^\pm,P_{k\ell}^\pm\}$ which are either multivariate Gaussian or functions of Multivariate Gaussians. In order to state Algorithm \ref{algo:ml-mat-vamp_se}, we need to define certain random variables and functions appearing therein which are described below. Let $\Lodd=\{1,3,\ldots,L-1\}$ and $\Leven=\{2,4,\ldots,L-2\}$.

Define $\{\Thetabfbar_{k\ell}^\pm\}$ similar to $\Thetabf_{k\ell}^\pm$ from equation \eqref{eq:thetadef}  using $\{\Gammabfbar_{k\ell}^\pm\}$. Further, for $\ell=1,2,\ldots,L-1$ define
\begin{equation*}\label{eq:omegadef}
\Omegabfbar_{k\ell}^+:= (\Lambdabfbar_{k\ell}^+,\Gammabfbar_{k\ell}^+,\Gammabfbar_{k\ell}^-),
\ \Omegabfbar_{k\ell}^-:= (\Lambdabfbar_{k,\ell-1}^+,\Gammabfbar_{k,\ell-1}^-,\Gammabfbar_{k,\ell-1}^-),
\end{equation*}
and $\Omegabfbar_{k0}^+$ and $\Omegabfbar_{kL}^-.$ Now define random variables $W_\ell$ as
\begin{equation}\label{eq:Wdef}
\begin{aligned}
\MoveEqLeft W_0 = Z_0^0,\ \ 
 W_L=(Y,\Xi_L),
\ \  W_\ell=\Xi_\ell,&\forall\,\ell\in\Leven,\\
\MoveEqLeft W_\ell=(S_\ell,\wb B_\ell,\Xi_\ell),&\forall\,\ell\in\Lodd.
\end{aligned}
\end{equation}
Define functions $\{f_\ell^0\}_{\ell=1}^L$ as
\begin{equation}\label{eq:f0def}
\begin{aligned}
\MoveEqLeft     f^0_\ell(P^0_{\lm1},W_\ell) := {S}_\ell P^0_{\ell-1} + \wb{B}_\ell + {\Xi}_\ell, 
    \quad \forall\,\ell\in\Lodd,
    \\
\MoveEqLeft f^0_\ell(P^0_{\lm1},W_\ell) := \phi_\ell(P^0_{\lm1},\Xi_\ell), 
    \quad \forall\,\ell\in\Leven\cup\{L\}. 
\end{aligned}
\end{equation}
and using \eqref{eq:thetadef} define functions $\{h_\ell^\pm,\}_{\ell=1}^L$, $h_0^+$ and $h_L^-$ as
\begin{equation}\label{eq:hdef}
\begin{aligned}
\MoveEqLeft h_\ell^\pm(P_{\ell-1}^0,P_{\ell-1}^+,Q_\ell^-,W_\ell,\Thetabf_{k\ell}^\pm) \\
\MoveEqLeft\  = G_\ell^\pm(Q_\ell^-+Q_{\ell}^0,P_{\ell-1}^++P_{\ell-1}^0,\Thetabf_{k\ell}^\pm),\ \  \forall\,\ell\in\Leven,\\
\MoveEqLeft h_\ell^\pm(P_{\ell-1}^0,P_{\ell-1}^+,Q_\ell^-,W_\ell,\Thetabf_{k\ell}^\pm) \\
\MoveEqLeft\ = \wt{G}_\ell^\pm(Q_\ell^-+Q_{\ell}^0,P_{\ell-1}^++P_{\ell-1}^0,\Thetabf_{k\ell}^\pm),\ \  \forall\,\ell\in\Lodd\\
\MoveEqLeft h_0^+(Q_{0}^-,W_{0},\Thetabf_{k0}^+) = G^+_0(Q_0^-+W_{0},\Thetabf_{k0}^+),
\\ 
\MoveEqLeft h_{\scaleto{L}{4pt}}^-(P_{\scaleto{L-1}{4pt}}^0,P_{\scaleto{L-1}{4pt}}^+,W_{\scaleto{L}{4pt}},\Thetabf_{{\scaleto{kL}{4pt}}}^-) = G^-_{\scaleto{L}{4pt}}(P_{\scaleto{L-1}{5pt}}^++P^0_{\scaleto{L-1}{5pt}},\Thetabf_{{\scaleto{kL}{4pt}}}^-).\end{aligned}
\end{equation}
Note that $[G_\ell^+,G_\ell^-]$ and $[\wt G_\ell^+,\wt G_\ell^-]$ are maps from $\Real^{1\times d}\rightarrow \Real^{1\times d}$ such that their row-wise extensions are the denoisers $[\Gbf_\ell^+,\Gbf_\ell^-]$ and $[\wt\Gbf_\ell^+,\wt\Gbf_\ell^-]$ respectively.
Using \eqref{eq:hdef} define functions $\{f_\ell^\pm\}_{\ell=1}^{L-1}$, $f_0^+$ and $f_L^-$ as
\begin{equation}\label{eq:fdef}
\begin{aligned}
\MoveEqLeft f^{+}_{\ell}(P^0_{\lm1},P_{\lm1}^+,Q_{\ell}^-,W_\ell,\Omegabf_{k\ell}^+)
\\
\MoveEqLeft = 
      \left[
      \left(h^{+}_{\ell}
- Q^0_\ell\right)\Lambdabf_{k\ell}^+ - Q_\ell^-\Gammabf_{k\ell}^- \right](\Gammabf_{k\ell}^+)^{-1},  \\
\MoveEqLeft f^{-}_{\ell}(P^0_{\lm1},P_{\lm1}^+,Q_{\ell}^-,W_\ell,\Omegabf_{k\ell}^-)
\\
\MoveEqLeft = 
     \left[
            \left(h^{-}_{\ell}
- P^0_{\lm1}\right)\Lambdabf_{k,\lm1}^-
            -  P_{\lm1}^+\Gammabf_{k,\lm1}^+ \right](\Gammabf_{k,\lm1}^-)^{-1}. \\
\MoveEqLeft f^{+}_{0}(Q_{0}^-,W_0,\Omegabf_{k0}^+) \\
\MoveEqLeft= 
     \left[
            \left(h^{+}_{0}
-W_0\right)\Lambdabf_{k0}^+
            -  Q_0^-\Gammabf_{k0}^- \right](\Gammabf_{k0}^+)^{-1}, \\
\MoveEqLeft      f^{-}_{L}(P^0_{\Lm1},P_{\Lm1}^+,W_L,\Omegabf_{kL}^-)\\
\MoveEqLeft     = 
     \left[
            \left(h^{-}_{L}
- P^0_{\Lm1}\right)\Lambdabf_{k,\Lm1}^-
            - P_{\Lm1}^+\Gammabf_{k,\Lm1}^+ \right](\Gammabf_{k,\Lm1}^-)^{-1}.    
\end{aligned}
\end{equation}

\section{Large System Limit Details}
\label{sec:lsl_details}
The analysis of Algorithm \ref{algo:ml-mat-vamp} in the large system limit is based on \cite{BayatiM:11} and is by now standard in the theory of AMP-based algorithms. The goal is to characterize ensemble row-wise averages of iterates of the algorithm using \textit{simpler} finite-dimensional random variables which are either Gaussians or functions of Gaussians.
To that end, we start by defining some key terms needed in this analysis.
\begin{definition}[Pseudo-Lipschitz continuity]
For a given $p\geq 1$, a map $\gbf:\Real^{1\times d}\rightarrow \Real^{1\times r}$ is called pseudo-Lipschitz of order $p$ if for any $\rbf_1,\rbf_2\in\Real^d$ we have,
\begin{align*}
\|\gbf(\rbf_1)-\gbf(\rbf_2)\|\leq C\|\rbf_1-\rbf_2\|\left(1+\|\rbf_1\|^{p-1}+\|\rbf_2\|^{p-1}\right)
\end{align*}
\end{definition}

\newcommand{\PLT}{\rm PL(2)}
\begin{definition}[Empirical convergence of rows of a matrix sequence] Consider a matrix-sequence $\{\Xbf^{(N)}\}_{N=1}^\infty$ with $\Xbf^{(N)}\in\Real^{N\times d}.$ For a finite $p\geq 1$, let $X\in(\Real^d,\mc R^d)$ be a $\mc R^d$-measurable random variable with bounded moment $\Exp\|X\|_p^p <\infty$. We say the rows of matrix sequence $\{\Xbf^{(N)}\}$ {\it converge empirically to $X$ with $p^{th}$ order moments} if
for all pseudo-Lipschitz continuous functions $f(\cdot)$ of order $p$,
\begin{align}\label{eq:Wasserstein_p}
\lim_{N\rightarrow \infty}\frac1N\sum_{n=1}^N f(\Xbf^{(N)}_{n:})=\Exp[f(X)]\quad {\rm a.s.}
\end{align}
\end{definition}
Note that the sequence $\{\Xbf^{(N)}\}$ could be random or deterministic. If it is random, however, then the quantities on the left hand side are random sums and the almost sure convergence must take this randomness into account as well.

The above convergence is equivalent to requiring weak convergence  as well as convergence of the $p^{\rm th}$ moment, of the empirical distribution $\tfrac1N\sum_{n=1}^N\delta_{\Xbf_{n:}^{(N)}}$ of the rows of $\Xbf^{(N)}$ to $X$. This is also referred to convergence in the Wasserstein-$p$ metric \cite[Chap. 6]{villani2008optimal}.

In the case of $p=2$, the condition is equivalent to requiring \eqref{eq:Wasserstein_p} to hold for all continuously bounded functions $f$ as well as for all $f_q(\x)=\x\tran\bm{Q}\x$ for all positive definite matrices $\bm{Q}$.

\begin{definition}[Uniform Lipschitz continuity] For a positive definite matrix $\M$, the map $\phi(\rbf;\M):\Real^{d}\rightarrow \Real^d$ is said to be uniformly Lipschitz continuous in $\rbf$ at $\M=\wb\M$ if there exist non-negative constants $L_1$, $L_2$ and $L_3$ such that for all $\rbf\in\Real^{d}$
\begin{align*}
\|\phi(\rbf_1;\M_0)-\phi(\rbf_2;\M_0)\| &\leq L_1\|\rbf_1-\rbf_2\|\\
\|\phi(\rbf;\M_1)-\phi(\rbf;\M_2)\| &\leq L_2(1+\|\rbf\|)\rho(\M_1,\M_2)
\end{align*}
for all $\M_i$ such that $\rho(\M_i,\wb\M)<L_3$
where $\rho$ is a metric on the cone of positive semidefinite matrices.
\end{definition}
We are now ready to prove Theorem \ref{thm:main_result}. 
\section{Proof of Theorem \ref{thm:main_result}}

The proof of Theorem \ref{thm:main_result} is a special case of a more general result on multi-layer recursions given in Theorem \ref{thm:general_convergence}. This result is stated in \ref{app:general_convergence}, and proved in \ref{app:proof_of_general_convergence}. The rest of this section identifies certain relevant quantities from Theorem \ref{thm:main_result} in order to apply Theorem \ref{thm:general_convergence}.

Consider the SVD given in equation \eqref{eq:svd} of weight matrices $\Wbf_\ell$ of the network.
We analyze Algo. \ref{algo:ml-mat-vamp} using \textit{transformed} versions of the true signals $\Zbf_\ell^0$ and input errors $\Rbf_\ell^\pm-\Zbf_\ell^0$ to the denoisers $\Gbf_\ell^\pm$. For $\ell=0,2,\ldots L-2$, define
\begin{subequations}\label{eq:PQ0def}
\begin{align}
    \MoveEqLeft \qbf_\ell^0 = \Zbf_\ell^0 
    &\qbf_{\ell+1}^0 &= \Vbf_{\ell+1}\T\Zbf_{\ell+1}^0\\
    \MoveEqLeft
    \pbf_{\ell}^0 = \Vbf_{\ell}\Zbf_{\ell}^0 &\pbf_{\ell+1}^0 &= \Zbf_{\ell+1}^0
\end{align}
\end{subequations}
which are depicted in Fig. \ref{fig:mlmatvamp_equivalent_system} (TOP). 
Similarly, define the following \textit{transformed} versions of errors in the inputs $\Rbf_\ell^\pm$ to the denoisers $\Gbf_\ell^\pm$
\begin{subequations}\label{eq:PQin_def}
\begin{align}
\MoveEqLeft\qbf_\ell^- = \Rbf_\ell^--\Zbf_\ell^0
    &\qbf_{\ell+1}^- &= \Vbf_{\ell+1}\T(\Rbf_{\ell+1}^--\Zbf_{\ell+1}^0)\\
    \MoveEqLeft
    \pbf_\ell^+ = \V_\ell(\Rbf_\ell^+-\Zbf_\ell^0)
    &\pbf_{\ell+1}^+ &= \Rbf_{\ell+1}^+-\Zbf_{\ell+1}^0
\end{align}
\end{subequations}
These quantities are depicted as inputs to function blocks $\fbf_\ell^\pm$ in Fig. \ref{fig:mlmatvamp_equivalent_system}
(MIDDLE). Define perturbation variables $\wbf_\ell$ as
\begin{subequations}\label{eq:wdef}
\begin{align}
\wbf_0 &= \Zbf^0_0,\ \ 
\wbf_{\scaleto{L}{4pt}} = (\Ybf,\Xibf_{\scaleto{L}{4pt}}),\ \  \wbf_\ell = \Xibf_\ell, &\forall\,\ell\in\Leven
\\
\wbf_\ell &= (\Sbf_\ell,\wb\Bbf_\ell,\Xibf_\ell), &\forall\,\ell\in\Lodd
\end{align}
\end{subequations}
Finally, we define $\qbf_\ell^+$ and $\pbf_\ell^-$ for $\ell=1,2,\ldots, L-1$ as
\begin{subequations}\label{eq:PQout_def}
\begin{align}
     \qbf_\ell^+ &= \fbf_\ell^+(\pbf_{\ell-1}^0,\pbf_{\ell-1}^+,\qbf_\ell^-,\wbf_\ell,\Omega_\ell)\\
     \pbf_{\lm1}^- &= \fbf_\ell^-(\pbf_{\ell-1}^0,\pbf_{\ell-1}^+,\qbf_\ell^-,\wbf_\ell,\Omega_\ell),
\end{align}
\end{subequations}
which are outputs of function blocks in Fig. \ref{fig:mlmatvamp_equivalent_system} (MIDDLE). Similarly, define the quantities $\qbf_0^+=\fbf_0^+(\qbf_0^-,\Zbf_0,\Omega_0)$ and 
$\pbf_{L-1}^-=\fbf_L^+(\pbf_{L-1}^0,\pbf_{L-1}^+,\Ybf,\Omega_L)$.

\begin{lemma}\label{lem:special_case}
Algorithm \ref{algo:ml-mat-vamp} is a special case of Algorithm \ref{algo:gen} with the definitions $\{\qbf_\ell^0,\pbf_\ell^0,\qbf_\ell^\pm,\pbf_\ell^\pm\}_{\ell=0}^{L-1}$ given in equations \eqref{eq:PQ0def},\eqref{eq:PQin_def}, and \eqref{eq:PQout_def}, functions $\fbf_{\ell}^\pm$ are row-wise extensions of $f_\ell^\pm$ defined using equations \eqref{eq:fdef} and \eqref{eq:hdef}.
\end{lemma}

\begin{lemma}
Assumptions \ref{as:gen} and \ref{as:gen2} required for applying Theorem \ref{thm:general_convergence} are satisfied by the conditions in Theorem \ref{thm:main_result}.
\end{lemma}

\begin{proof} The proofs of the above lemmas are identical to the case of $d=1$, which was shown in \cite{pandit2019inference}. For details
see \cite[Appendix F]{pandit2019inference}.
\end{proof} 
\section{General Multi-Layer Recursions}
\label{app:general_convergence}
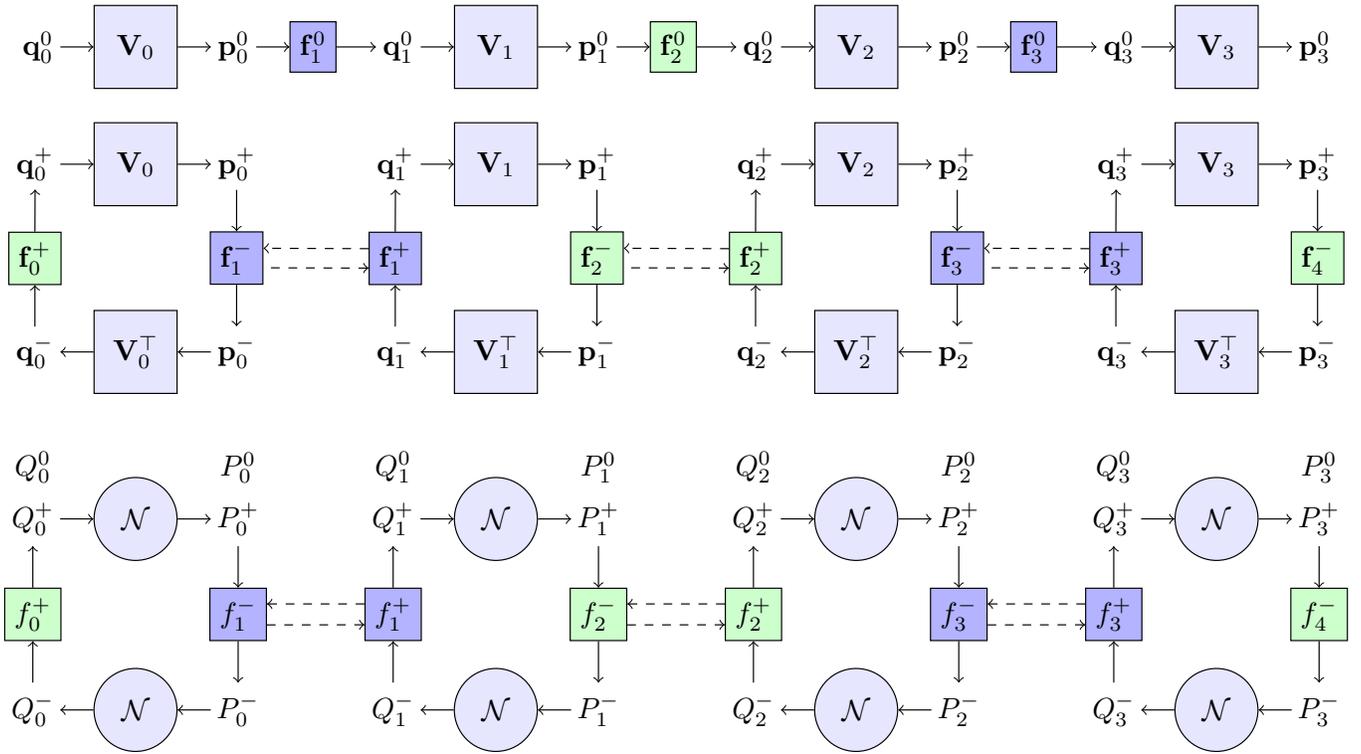
\begin{figure*}[ht]
\resizebox{\textwidth}{!}{
\begin{tikzpicture}[scale = 0.9]

    \pgfmathsetmacro{\sep}{3};
    \pgfmathsetmacro{\yoff}{0.4};
    \pgfmathsetmacro{\xoffa}{0.3};
    \pgfmathsetmacro{\xoffb}{0.6};
    \tikzstyle{var}=[draw=white,circle,fill=white!100,node distance=0cm,inner sep=0cm];
    \tikzstyle{rot}=[draw,fill=blue!10, minimum size=1cm, node distance=\sep cm];
    \tikzstyle{Dnl}=[draw,fill=green!20, minimum size=.5cm, node distance=\sep cm];
    \tikzstyle{Dl}=[draw,fill=blue!30, minimum size=.5cm, node distance=\sep cm];
    \tikzstyle{ROT}=[draw,circle,fill=blue!10, minimum size=1cm, node distance=\sep cm];
    \tikzstyle{dnl}=[draw,fill=green!20, minimum size=.5cm, node distance=\sep cm];
    \tikzstyle{dl}=[draw,fill=blue!30, minimum size=.5cm, node distance=\sep cm];

\newcommand{\Vbm}{\bm V}
\newcommand{\pbfbm}{\bm P}
\newcommand{\qbfbm}{\bm Q}
\newcommand{\ubm}{\bm U}
\newcommand{\tbm}{\bm T}
\newcommand{\Abm}{\bm A}
\def\xoff{.4cm}
\def\xoffblock{1.3cm}
\def\yoff{.5cm}
\def\yspace{.4cm}
\def\Yspace{1cm}
\def\yvec{0mm}

\node [var] (q0p) {$\qbf^+_0$};
        
        \def\loopend{4}
\foreach \i/\j in {0/1,1/2,2/3,3/\loopend} {
    \node [rot, right =\xoff of q\i p] (V\i ) {$\V_\i $};
        {
        \node [rot, above = \yspace of V\i] (V\i star){$\V_\i$};
        \node [var, right=\xoff of V\i star] (p\i star) {$\pbf_\i ^0$};
        \ifthenelse{\j<\loopend}{
            \ifthenelse{\intcalcMod{\j}{2}>0}{
                \node [Dl, right =\xoff of p\i star] (phi\j) {$\fbf_\j^0$};
                }
                {
                \node [Dnl, right =\xoff of p\i star] (phi\j) {$\fbf_\j^0$};
            }
        }{}
        \node [var, left=\xoff of V\i star] (q\i star) {$\qbf_\i ^0$};
        }
        
    \node [var, right=\xoff of V\i ] (p\i p) {$\pbf_\i ^+$};
    
    \ifthenelse{\intcalcMod{\j}{2}>0}{
    \node [Dl, below =\yoff of p\i p] (F\j m) {$\fbf_\j^-$};
    }
    {
    \node [Dnl, below =\yoff of p\i p] (F\j m) {$\fbf_\j^-$};
    }
    
    \node [var, below=\yoff of F\j m] (p\i m) {$\pbf_\i ^-$};
	\node [rot, left =\xoff of p\i m] (V\i T) {$\V_\i \T$};
	    {
	    \node [ROT, below = \Yspace of V\i T] (N\i){$\mc N$};
	    \node [var, right=\xoff of N\i ] (P\i p) {$P_\i ^+$};
	    \node [var, above=\yvec of P\i p](P\i 0){$P_\i ^0$};
        \ifthenelse{\intcalcMod{\j}{2}>0}{
        \node [dl, below =\yoff of P\i p] (f\j m) {$f_\j^-$};
        }
        {
        \node [dnl, below =\yoff of P\i p] (f\j m) {$f_\j^-$};
        }
        
        \node [var, below=\yoff of f\j m] (P\i m) {$P_\i ^-$};
    	\node [ROT, left =\xoff of P\i m] (N\i T) {$\mcN$};
        \node [var, left =\xoff of N\i T] (Q\i m) {$Q^-_\i $};
        \ifthenelse{\intcalcMod{\j}{2}>0}{
        \node [dnl, above =\yoff of Q\i m] (f\i p) {$f_\i^+$};
        }
        {	
        \node [dl, above =\yoff of Q\i m] (f\i p) {$f_\i^+$};
        }
    	\node [var, above =\yoff of f\i p] (Q\i p) {$Q^+_\i $};
    	\node [var, above = \yvec of Q\i p](Q\i 0){$Q^0_\i$};
	    }

    \node [var, left =\xoff of V\i T] (q\i m) {$\qbf^-_\i $};
    \ifthenelse{\intcalcMod{\j}{2}>0}{
    \node [Dnl, above =\yoff of q\i m] (F\i p) {$\fbf_\i^+$};
    }
    {	
    \node [Dl, above =\yoff of q\i m] (F\i p) {$\fbf_\i^+$};
    }
    \ifthenelse{\j<\loopend}{
	\node [var, right =\xoffblock of p\i p] (q\j p) {$\qbf^+_\j$};
	}{}
}

\foreach \i/\j in {0/1,1/2,2/3,3/\loopend} {

\draw[->] (q\i star.east) -- (V\i star.west);
\draw[->] (V\i star.east) -- (p\i star.west);
\ifthenelse{\j<\loopend}{
    \draw[->] (p\i star.east) -- (phi\j.west);
    \draw[->] (phi\j.east) -- (q\j star.west);
}

\draw[->] (q\i p.east) -- (V\i.west);
\draw[->] (V\i.east) -- (p\i p.west);
\draw[->] (p\i p.south) -- (F\j m.north);
\draw[->] (F\j m.south) -- (p\i m.north);
\draw[->] (p\i m.west) -- (V\i T.east);
\draw[->] (V\i T.west) -- (q\i m.east);
\draw[->] (q\i m.north) -- (F\i p.south);
\draw[->] (F\i p.north) -- (q\i p.south);

\ifthenelse{\i>0}{
\draw[->,dashed] (F\i p.160) -- (F\i m.20);
\draw[<-,dashed] (F\i p.-160) -- (F\i m.-20);
\draw[->,dashed] (f\i p.160) -- (f\i m.20);
\draw[<-,dashed] (f\i p.-160) -- (f\i m.-20);
}{}

\draw[->] (Q\i p.east) -- (N\i.west);
\draw[->] (N\i.east) -- (P\i p.west);
\draw[->] (P\i p.south) -- (f\j m.north);
\draw[->] (f\j m.south) -- (P\i m.north);
\draw[->] (P\i m.west) -- (N\i T.east);
\draw[->] (N\i T.west) -- (Q\i m.east);
\draw[->] (Q\i m.north) -- (f\i p.south);
\draw[->] (f\i p.north) -- (Q\i p.south);

}
\end{tikzpicture}

 }
\caption{\label{fig:mlmatvamp_equivalent_system}
(TOP) The equations \eqref{eq:nntrue} with equivalent quantities defined in \eqref{eq:PQ0def}, and $\fbf_\ell^0$ defined using \eqref{eq:f0def}.\newline
(MIDDLE) The Gen-ML-Mat recursions in Algorithm \ref{algo:gen}. These are also equivalent to ML-Mat-VAMP recursions from Algorithm \ref{algo:ml-mat-vamp} (See Lemma \ref{lem:special_case}) if $\qbf^\pm,\pbf^\pm$ are as defined as in equations \eqref{eq:PQin_def} and \eqref{eq:PQout_def}, and $\fbf_\ell^\pm$ given by equations \eqref{eq:fdef} and $\eqref{eq:hdef}$.\newline
(BOTTOM) Quantities in the GEN-ML-SE recursions. These are also equivalent to ML-Mat-VAMP SE recursions from Algorithm \ref{algo:ml-mat-vamp_se} (See Lemma \ref{lem:special_case})\newline
The iteration indices $k$ have been dropped for notational simplicity.
}
\end{figure*}

To analyze Algorithm~\ref{algo:ml-mat-vamp}, we consider a more general class
of recursions as given in Algorithm~\ref{algo:gen} and depicted in Fig. \ref{fig:mlmatvamp_equivalent_system}.
The Gen-ML recursions generates
(i) a set of \textit{true matrices} $\q_\ell^0$ and $\p_\ell^0$ 
and (ii) \textit{iterated matrices} $\qbf^{\pm}_{k\ell}$ and $\pbf^{\pm}_{k\ell}$. Each of these matrices have the same number of columns, denoted by $d$. 

The true matrices are generated by a single forward pass, whereas the iterated matrices are generated
via a sequence of forward and backward passes through a multi-layer system.
In proving the State Evolution for the ML-Mat-VAMP algorithm (Algo. \ref{algo:ml-mat-vamp}, one would then associate the terms $\qbf^{\pm}_{k\ell}$ and $\pbf^{\pm}_{k\ell}$
with certain error quantities in the ML-Mat-VAMP recursions. To account for the effect of the parameters $\Gammabf^{\pm}_{k\ell}$ and $\Lambdabf^{\pm}_{k\ell}$
in ML-Mat-VAMP, the Gen-ML algorithm describes the parameter updates through a sequence of
\emph{parameter lists} $\Upsilon^{\pm}_{k\ell}$.
The parameter lists are ordered lists of parameters that accumulate as the algorithm progresses. The true and iterated matrices from Algorithm \ref{algo:gen} are depicted in the signal flow graphs on the (TOP) and (MIDDLE) panel of Fig. \ref{fig:mlmatvamp_equivalent_system} respectively. The iteration index $k$ for the iterated vectors $\q_{k\ell},\p_{k\ell}$ has been dropped for simplifying notation.

The functions $\fbf_\ell^0(\cdot)$ that produce the true matrices $\q_\ell^0,\p_\ell^0$ are called \textit{initial matrix functions} and use the initial parameter list $\Upsilon_{01}^-$. The functions $\fbf_{k\ell}^{\pm}(\cdot)$ that produce the matrices
$\qbf^{+}_{k\ell}$ and $\pbf^{-}_{k\ell}$ are  called the \emph{matrix update functions} and use parameter lists $\Upsilon_{kl}^\pm$.
The initial parameter lists $\Upsilon^-_{01}$ are assumed to be provided. 
As the algorithm progresses, new parameters $\lambda^{\pm}_{k\ell}$
are computed and then added to the lists in lines~\ref{line:lamp0_gen}, \ref{line:lamp_gen}, \ref{line:lamL_gen}
and \ref{line:lamn_gen}.  The matrix update functions $\fbf_{k\ell}^{\pm}(\cdot)$ may depend on any sets of parameters accumulated in the parameter list. 
In lines~\ref{line:mup0_gen}, \ref{line:mup_gen}, \ref{line:muL_gen} and \ref{line:mun_gen},
the new parameters $\lambda_{k\ell}^{\pm}$ are computed by:
(1) computing average values $\mu_{k\ell}^{\pm}$ of \emph{row-wise} functions $\varphibf^{\pm}_{k\ell}(\cdot)$;
and (2) taking functions $T^{\pm}_{k\ell}(\cdot)$ of the average values $\mu_{k\ell}^{\pm}$.
Since the average values $\mu_{k\ell}^{\pm}$ represent statistics on the rows of
$\varphibf^{\pm}_{k\ell}(\cdot)$, we will call $\varphibf^{\pm}_{k\ell}(\cdot)$ the \emph{parameter statistic
functions}.  We will call the $T^{\pm}_{k\ell}(\cdot)$ the \emph{parameter update functions}.
The functions $\fbf_\ell^0,\fbf_{k\ell}^\pm,\varphibf^\pm_\ell$ also take as input some perturbation vectors $\w_\ell$.

\old{We will show below that the updates for the parameters $\gamma^{\pm}_{k\ell}$ and $\alpha^{\pm}_{k\ell}$
can be written in this form.}

\begin{algorithm}[t]
\setstretch{1.1}
\caption{General Multi-Layer Matrix (Gen-ML-Mat) Recursion }
\begin{algorithmic}[1]  \label{algo:gen}
\REQUIRE{Initial matrix functions $\{\fbf_\ell^0\}$. Matrix update functions $\{\fbf^\pm_{k\ell}(\cdot)\}$.
Parameter statistic functions $\{\varphibf^\pm_{k\ell}(\cdot)\}$.
Parameter update functions $\{T^{\pm}_{k\ell}(\cdot)\}$.
Orthogonal matrices $\{\Vbf_\ell\}$.
Perturbation variables $\{\wbf^\pm_\ell\}$. Initial matrices $\{\qbf_{0\ell}^-\}$. Initial parameter list $\Upsilon_{01}^-$.}

\STATE{// \texttt{Initial Pass} }
\STATE{$\qbf^0_0 = \fbf^0_0(\wbf_0), \quad \pbf^0_0 = \Vbf_0\qbf^0_0$} \label{line:q00init_gen}
\FOR{$\ell=1,\ldots,\Lm1$}
    \STATE{$\qbf^0_\ell = \fbf^0_\ell(\pbf^0_{\lm1},\wbf_\ell, \Upsilon_{01}^-)$ }
    \label{line:q0init_gen}
    \STATE{$\pbf^0_\ell = \Vbf_\ell\qbf^0_\ell$ }  \label{line:p0init_gen}
\ENDFOR \label{line:end_initial_for}
\STATE{}
\FOR{$k=0,1,\dots$}\label{line:start_algo_for}
    \STATE{// \texttt{Forward Pass} }
    \STATE{$\lambda^+_{k0} = T_{k0}^+(\mu^+_{k0},\Upsilon_{0k}^-)$}
    \STATE{$\mu^+_{k0} = \bkt{\varphibf_{k0}^+(\qbf_{k0}^-,\wbf_0,\Upsilon_{0k}^-)}$}    \label{line:mup0_gen}
    \STATE{$\Upsilon_{k0}^+ = (\Upsilon_{k1}^-,\lambda^+_{k0})$} \label{line:lamp0_gen}
    \STATE{$\qbf_{k0}^+ = \fbf^+_{k0}(\qbf_{k0}^-,\wbf_0,\Upsilon^+_{k0})$}  \label{line:q0_gen}
    \STATE{$\pbf_{k0}^+ = \Vbf_0\qbf_{k0}^+$} \label{line:p0_gen}
    \FOR{$\ell=1,\ldots,L-1$}
        \STATE{$\lambda^+_{k\ell} = T_{k\ell}^+(\mu^+_{k\ell},\Upsilon_{k,\lm1}^+)$}
    \STATE{$
            \mu^+_{k\ell} = \bkt{\varphibf_{k\ell}^+(\pbf^0_{\lm1},\pbf^+_{k,\lm1},\qbf_{k\ell}^-,\wbf_\ell,\Upsilon_{k,\lm1}^+)}$}    \label{line:mup_gen}
        \STATE{$\Upsilon_{k\ell}^+ = (\Upsilon_{k,\lm1}^+,\lambda^+_{k\ell})$}
            \label{line:lamp_gen}
        \STATE{$\qbf_{k\ell}^+ = \fbf^+_{k\ell}(\pbf^0_{\lm1},\pbf^+_{k,\lm1},\qbf_{k\ell}^-,\wbf_\ell,\Upsilon^+_{k\ell})$}
            \label{line:qp_gen}
        \STATE{$\pbf_{k\ell}^+ = \Vbf_{\ell}\qbf_{k\ell}^+$}   \label{line:pp_gen}
    \ENDFOR
\vspace{10pt}
    \STATE{// \texttt{Backward Pass} }
    \STATE{$\lambda^-_{\kp1,L} = T_{kL}^-(\mu^-_{kL},\Upsilon_{k,\Lm1}^+)$}
    \STATE{$
        \mu^-_{kL} = \bkt{\varphibf_{kL}^-(\pbf_{k,\Lm1}^+,\wbf_L,\Upsilon_{k,\Lm1}^+)}$}    \label{line:muL_gen}
    \STATE{$\Upsilon_{\kp1,L}^- = (\Upsilon_{k,\Lm1}^+,\lambda^+_{\kp1,L})$} \label{line:lamL_gen}
    \STATE{$\pbf_{\kp1,\Lm1}^- = \fbf^-_{kL}(\pbf^0_{\Lm1},\pbf_{k,\Lm1}^+,\wbf_L,\Upsilon^-_{\kp1,L})$}  \label{line:pL_gen}
    \STATE{$\qbf_{\kp1,\Lm1}^- = \Vbf_{\Lm1}\tran\pbf_{\kp1,\Lm1}$} \label{line:qL_gen}
    \FOR{$\ell=\Lm1,\ldots,1$}
        \STATE{$\lambda^-_{\kp1,\ell} = T_{k\ell}^-(\mu^-_{k\ell},\Upsilon_{\kp1,\lp1}^-)$}
    \STATE{$
            \mu^-_{k\ell} =
            \bkt{\varphibf_{k\ell}^-(\pbf_{\lm1}^0,\pbf_{k,\lm1}^+,\qbf_{\kp1,\ell}^-,\wbf_\ell,\Upsilon_{\kp1,\lp1}^-)}$}    \label{line:mun_gen}
        \STATE{$\Upsilon_{\kp1,\ell}^- = (\Upsilon_{\kp1,\lp1}^-,\lambda^-_{\kp1,\ell})$} \label{line:lamn_gen}
        \STATE{$\pbf_{\kp1,\lm1}^- =
        \fbf^-_{k\ell}(\pbf_{\lm1}^0,\pbf^+_{k,\lm1},\qbf_{\kp1,\ell}^-,\wbf_\ell,\Upsilon^-_{k+1,\ell})$}
            \label{line:pn_gen}
        \STATE{$\qbf_{\kp1,\lm1}^- = \Vbf_{\lm1}\tran\pbf_{\kp1,\lm1}^-$}   \label{line:qn_gen}
    \ENDFOR

\ENDFOR\label{line:end_algo_for}
\end{algorithmic}
\end{algorithm}

Similar to the analysis of the ML-Mat-VAMP Algorithm,
we consider the following large-system limit (LSL) analysis of Gen-ML.
Specifically, we consider a sequence of runs of the recursions indexed by $N$.
For each $N$, let $N_\ell = N_\ell(N)$ be the dimension of the matrix signals $\pbf_\ell^{\pm}$ and $\qbf_\ell^\pm$
as we assume that $\displaystyle\lim_{N \arr \infty} \tfrac{N_\ell}N = \beta_\ell\in(0,\infty)$ is a constant so that $N_\ell$ scales linearly with $N$. Note however that the number of columns of each of the matrices $\{\qbf_\ell^0,\pbf_\ell^0,\qbf_{k\ell}^\pm,\pbf_{k\ell}^\pm\}$ is equal to a finite integer $d>0,$ which remains fixed for all $N$.
We then make the following assumptions. See \ref{sec:lsl_details} for an overview of empirical convergence of sequences which we use in the assumptions described below.

\begin{assumption}\label{as:gen} For vectors in the Gen-ML Algorithm (Algorithm~\ref{algo:gen}),
we assume:
\begin{enumerate}[(a)]
\item\label{as1:a} The matrices $\Vbf_\ell$ are Haar distributed on the set of $N_\ell \times N_\ell$ orthogonal matrices and are
independent from one another and from the matrices $\q^0_0$,
$\qbf_{0\ell}^-$, perturbation variables $\wbf_\ell$.
\item\label{as1:b} The rows of the initial matrices
$\qbf_{0\ell}^-$, and perturbation variables $\wbf_\ell$ converge jointly empirically with limits,
\beq \label{eq:qwinitlim}
    \qbf_{0\ell}^- \xRightarrow{2} Q_{0\ell}^-, \quad
    \wbf_{\ell} \xRightarrow{2} W_\ell,
\eeq
where $Q_{0\ell}^-$ are random vectors in $\Real^{1\times d}$ such that $(Q_{00}^-,\cdots,Q^-_{0,\Lm1})$
is jointly Gaussian. For $\ell=0,\ldots,\Lm1$, the random variables $W_\ell, P_{\ell-1}^0$ and $Q_{0\ell}^-$ are all independent.
We also assume that the initial parameter list converges as
\begin{align} \label{eq:Lambar01lim}
    \lim_{N \arr \infty} \Upsilon_{01}^-(N) \xrightarrow{a.s.} \Upsilonbar_{01}^-,
\end{align}
to some list $\Upsilonbar_{01}^-$.  The limit \eqref{eq:Lambar01lim} 
means that every element in the list $\lambda(N) \in \Upsilon_{01}^-(N)$ converges to a limit
$\lambda(N) \arr \lambdabar\in\Upsilonbar_{01}^-$ as $N \arr \infty$ almost surely.

\item\label{as1:c} The \textit{matrix update functions} $\fbf_{k\ell}^\pm(\cdot)$
and \textit{parameter update functions} $\varphibf_{k\ell}^\pm(\cdot)$ act row-wise.  For e.g.,
in the $k^{\rm th}$ forward pass, at  stage $\ell$, we assume that for each output row $n$,
\begin{align*}
    \MoveEqLeft\left[ \fbf^+_{k\ell}(\pbf^0_{\lm1},\pbf^+_{k,\lm1},\qbf_{k\ell}^-,\wbf_\ell,\Upsilon^+_{k\ell}) \right]_{n:}\\
    \MoveEqLeft\qquad\qquad = f^+_{k\ell}(\pbf^0_{\lm1,n:},\pbf^+_{k,\lm1,n:},\qbf_{k\ell,n:}^-,\wbf_{\ell,n:},\Upsilon^+_{k\ell}) \end{align*}
    \begin{align*}
    \MoveEqLeft\left[ \varphibf^{+}_{k\ell}(\pbf^0_{\lm1},\pbf^+_{k,\lm1},\qbf_{k\ell}^-,\wbf_\ell,\Upsilon^+_{k\ell}) \right]_{n:}\\
    \MoveEqLeft\qquad\qquad = \varphi^+_{k\ell}(\pbf^0_{\lm1,n:},\pbf^+_{k,\lm1,n:},\qbf_{k\ell,n:}^-,\wbf_{\ell,n:},\Upsilon^+_{k\ell}),
\end{align*}
for some $\Real^{1\times d}$-valued functions $f^+_{k\ell}(\cdot)$ and $\varphi^+_{k\ell}(\cdot)$.
Similar definitions apply in the reverse directions and for the initial vector functions $\fbf^0_\ell(\cdot)$.
We will call $f^{\pm}_{k\ell}(\cdot)$ the \emph{matrix update row-wise
functions} and $\varphi^{\pm}_{k\ell}(\cdot)$ the \emph{parameter update row-wise functions}.
\end{enumerate}
\end{assumption}

\begin{algorithm}[t]
\setstretch{1.1}
\caption{Gen-ML-Mat State Evolution (SE)}
\begin{algorithmic}[1]  \label{algo:gen_se}

\REQUIRE{Matrix update row-wise functions $f^0_\ell(\cdot)$ and $f^\pm_{k\ell}(\cdot)$,
parameter statistic row-wise functions $\varphi^\pm_{k\ell}(\cdot)$,
parameter update functions $T^{\pm}_{k\ell}(\cdot)$, initial parameter list limit:  $\Upsilonbar_{01}^-$, initial random variables  $W_\ell$, $Q_{0\ell}^-$, $\ell=0,\ldots,\Lm1$.}

\STATE{// \texttt{Initial pass}}
\STATE{$Q^0_0 = f^0_0(W_0,\Upsilonbar_{01}^-), \quad P^0_0 \sim \Norm(0,\tau^0_0),
    \quad \tau^0_0 = \Exp(Q^0_0)^2$} \label{line:q0init_se_gen}
\FOR{$\ell=1,\ldots,\Lm1$}
    \STATE{$Q^0_\ell=f^0_\ell(P^0_{\lm1},W_\ell,\Upsilonbar_{01}^-)$}\label{line:Q0def}
    \STATE{$P^0_\ell \sim \Norm(0,\tau^0_\ell)$, \quad
            $\tau^0_\ell = \Cov(Q^0_\ell)$} \label{line:p0init_se_gen}
\ENDFOR
\STATE{}

\FOR{$k=0,1,\dots$}
    \STATE{// \texttt{Forward Pass }}
    \STATE{$\lambdabar^+_{k0} = T_{k0}^+(\mubar^+_{k0},\Upsilonbar_{0k}^-)$}
    \STATE{$
        \mubar^+_{k0} = \Exp(\varphi_{k0}^+(Q_{k0}^-,W_0,\Upsilonbar_{0k}^-))$}    \label{line:mup0_se_gen}
    \STATE{$\Upsilonbar_{k0}^+ = (\Upsilonbar_{k1}^-,\lambdabar^+_{k0})$} \label{line:lamp0_se_gen}
    \STATE{$Q_{k0}^+ = f^+_{k0}(Q_{k0}^-,W_0,\Upsilonbar^+_{k0})$}  \label{line:q0_se_gen}
    \STATE{$(P^0_0,P_{k0}^+) \sim \Norm(\zero,\Kbf_{k0}^+),
        \quad \Kbf_{k0}^+ = \Cov(Q^0_0,Q_{k0}^+)$} \label{line:p0_se_gen}
    \FOR{$\ell=1,\ldots,L-1$}
        \STATE{$\lambdabar^+_{k\ell} = T_{k\ell}^+(\mubar^+_{k\ell},\Upsilonbar_{k,\lm1}^+)$}
    \STATE{$
            \mubar^+_{k\ell} = \Exp(\varphi_{k\ell}^+(P^0_{\lm1},P^+_{k,\lm1},Q_{k\ell}^-,W_\ell,\Upsilonbar_{k,\lm1}^+))$}    \label{line:mup_se_gen}
        \STATE{$\Upsilonbar_{k\ell}^+ = (\Upsilonbar_{k,\lm1}^+,\lambdabar^+_{k\ell})$}
            \label{line:lamp_se_gen}
        \STATE{$Q_{k\ell}^+ = f^+_{k\ell}(P^0_{\lm1},P^+_{k,\lm1},Q_{k\ell}^-,W_\ell,\Upsilonbar^+_{k\ell})$}
            \label{line:qp_se_gen}
        \STATE{$(P^0_\ell,P_{k\ell}^+) \sim \Norm(\zero,\Kbf_{k\ell}^+), \quad
            \Kbf_{k\ell}^+ = \Cov(Q^0_\ell,Q_{k\ell}^+) $}   \label{line:pp_se_gen}
    \ENDFOR
    \vspace{10pt}

    \STATE{// \texttt{Backward Pass }}
    \STATE{$\lambdabar^-_{\kp1,L} = T_{kL}^-(\mubar^-_{kL},\Upsilonbar_{k,\Lm1}^+)$}
    \STATE{$
        \mubar^-_{kL} = \Exp(\varphi_{kL}^-(P^0_{\Lm1},P_{k,\Lm1}^+,W_L,\Upsilonbar_{k,\Lm1}^+))$}    \label{line:muL_se_gen}
    \STATE{$\Upsilonbar_{\kp1,L}^- = (\Upsilonbar_{k,\Lm1}^+,\lambdabar^+_{\kp1,L})$} \label{line:lamL_se_gen}
    \STATE{$P_{\kp1,\Lm1}^- = f^-_{kL}(P^0_{\Lm1},P_{k,\Lm1}^+,W_L,\Upsilonbar^-_{\kp1,L})$}  \label{line:pL_se_gen}
    \STATE{$Q_{\kp1,\Lm1}^- \sim \Norm(0,\tau_{\kp1,\Lm1}^-), \ \ 
        \tau_{\kp1,\Lm1}^- = \Cov(P^-_{\kp1,\Lm1})$} \label{line:qL_se_gen}
    \FOR{$\ell=\Lm1,\ldots,1$}
        \STATE{$\lambdabar^-_{\kp1,\ell} = T_{k\ell}^-(\mubar^-_{k\ell},\Upsilonbar_{\kp1,\lp1}^-)$}
    \STATE{$
            \mubar^-_{k\ell} =
         \Exp(\varphi_{k\ell}^-(P^0_{\lm1},P_{k,\lm1}^+,Q_{\kp1,\ell}^-,W_\ell,\Upsilonbar_{\kp1,\lp1}^-))$}    \label{line:mun_se_gen}
        \STATE{$\Upsilonbar_{\kp1,\ell}^- = (\Upsilonbar_{\kp1,\lp1}^-,\lambdabar^-_{\kp1,\ell})$} \label{line:lamn_se_gen}
        \STATE{$P_{\kp1,\lm1}^- =
        f^-_{k\ell}(P^0_{\lm1},P^+_{k,\lm1},Q_{\kp1,\ell}^-,W_\ell,\Upsilonbar^-_{k+1,\ell})$}
            \label{line:pn_se_gen}
        \STATE{$Q_{\kp1,\lm1}^- \sim \Norm(0,\tau_{\kp1,\lm1}^-), \ \ 
        \tau_{\kp1,\lm1}^- = \Cov(P_{\kp1,\lm1}^-)$}   \label{line:qn_se_gen}
    \ENDFOR

\ENDFOR
\end{algorithmic}
\end{algorithm} 

Next we define a set of \textit{deterministic} constants $\{\Kbf_{k\ell}^+,\taubf_{k\ell}^-,\wb{\mu}_{k\ell}^\pm,\wb\Upsilon_{kl}^\pm,\taubf_\ell^0\}$ and $\Real^{1\times d}$-valued random vectors $\{Q_\ell^0,P_\ell^0,Q_{k\ell}^\pm,P_{\ell}^\pm\}$ which are recursively defined through Algorithm~\ref{algo:gen_se}, which we call the \textit{Gen-ML-Mat State Evolution} (SE).
These recursions in Algorithm closely
mirror those in the Gen-ML-Mat algorithm (Algorithm~\ref{algo:gen}).  The matrices
$\qbf^\pm_{k\ell}$ and $\pbf^\pm_{k\ell}$ are replaced by random vectors
$Q^\pm_{k\ell}$ and $P^\pm_{k\ell}$; the matrix and parameter update functions
$\fbf^\pm_{k\ell}(\cdot)$ and $\varphibf^\pm_{k\ell}(\cdot)$ are replaced by their
row-wise functions $f^\pm_{k\ell}(\cdot)$ and $\varphi^\pm_{k\ell}(\cdot)$;
and the parameters $\lambda_{k\ell}^\pm$ are replaced
by their limits $\lambdabar_{k\ell}^\pm$. We refer to $\{Q_\ell^0,P_\ell^0\}$ as \textit{true random vectors} and $\{Q_{k\ell}^\pm,P_{kl}^\pm\}$ as \textit{iterated random vectors}. The signal flow graph for the true and iterated random variables in Algorithm \ref{algo:gen_se} is given in the (BOTTOM) panel of Fig. \ref{fig:mlmatvamp_equivalent_system}. The iteration index $k$ for the iterated random variables $\{Q_{k\ell}^\pm,P_{kl}^\pm\}$ to simplify notation.

We also assume the following about the behaviour of row-wise functions around the quantities defined in Algorithm \ref{algo:gen_se}.  The iteration index $k$ has been dropped for simplifying notation.

\begin{assumption} \label{as:gen2} For row-wise functions $f,\varphi$ and parameter update functions $T$ we assume:
\begin{enumerate}[(a)]
\item\label{as2:a} $T^\pm_{k\ell}(\mu_{k\ell}^\pm,\cdot)$ are continuous at
$\mu_{k\ell}^\pm = \mubar_{k\ell}^\pm$ \item\label{as2:b} $f^+_{k\ell}(p_{\ell-1}^0,p^+_{k,\lm1},q_{k\ell}^-,w_\ell,\Upsilon^+_{k\ell})$, $\tfrac{\partial f^+_{k\ell}}{\partial q_{k\ell}^-}(p_{\ell-1}^0,p^+_{k,\lm1},q_{k\ell}^-,w_\ell,\Upsilon^+_{k\ell})$ and $\varphi^+_{k\ell}(p_{\ell-1}^0,p^+_{k,\lm1},q_{k\ell}^-,w_\ell,\Upsilon^+_{k,\lm1})$ are uniformly Lipschitz continuous in $(p_{\ell-1}^0,p^+_{k,\lm1},q_{k\ell}^-,w_\ell)$ at
$\Upsilon^+_{k\ell} = \Upsilonbar^+_{k\ell}$, $\Upsilon^+_{k,\lm1} = \Upsilonbar^+_{k,\lm1}$. 
Similarly, \newline $f^-_{k+1,\ell}(p_{\ell-1}^0,p^+_{k,\lm1},q_{k+1,\ell}^-,w_\ell,\Upsilon^-_{k\ell}),$ $\tfrac{\partial f_{k\ell}^-}{\partial p_{k,\ell-1}^+}(p_{\ell-1}^0,p^+_{k,\lm1},q_{k+1,\ell}^-,w_\ell,\Upsilon^-_{k\ell}),$ and $\varphi^-_{k\ell}(p_{\ell-1}^0,p^+_{k,\lm1},q_{k+1,\ell}^-,w_\ell,\Upsilon^-_{k+1,\ell+1})$ are uniformly Lipschitz continuous in 
$(p_{\ell-1}^0,p^+_{k,\lm1},q_{k+1,\ell}^-,w_\ell)$ at $\Upsilon^-_{k\ell} = \Upsilonbar^-_{k\ell}$, $\Upsilon^-_{k+1,\ell+1} = \Upsilonbar^-_{k+1,\ell+1}$. 
\item\label{as2:c} $f^0_\ell(p^0_{\lm1},w_\ell,\Upsilon^-_{01})$ are uniformly Lipschitz
continuous in $(p^0_{k,\lm1},w_\ell)$ at $\Upsilon^-_{\kp1,\ell} = \Upsilonbar^-_{\kp1,\ell}$.
\item\label{as2:d} Matrix update functions $\fbf^\pm_{k\ell}$ are \emph{asymptotically divergence free} meaning
\beq \label{eq:fdivfree}
\begin{aligned}
    \MoveEqLeft\lim_{N \arr \infty} \bkt{\tfrac{\partial \fbf^+_{k\ell}}{
        \partial \qbf_{k\ell}^-}(\pbf^+_{k,\lm1},\qbf_{k\ell}^-,\wbf_\ell,\Upsilonbar^+_{k\ell})} = \zero,\\
    \MoveEqLeft\lim_{N \arr \infty} \bkt{\tfrac{\partial \fbf^-_{k\ell}}{
        \partial \pbf_{k,\lm1}^+} (\pbf^+_{k,\lm1},\qbf_{{k+1},\ell}^-,\wbf_\ell,\Upsilonbar^-_{k\ell})} = \zero
\end{aligned}
\eeq
\end{enumerate}
\end{assumption}

\medskip
We are now ready to state the general result regarding the empirical convergence of the true and iterated vectors from Algorithm \ref{algo:gen} in terms of random variables defined in Algorithm \ref{algo:gen_se}.

\begin{theorem} \label{thm:general_convergence}  Consider the iterates of the Gen-ML recursion (Algorithm~\ref{algo:gen})
and the corresponding random variables and parameter limits
defined by the SE recursions (Algorithm~\ref{algo:gen_se}) under Assumptions~\ref{as:gen} and \ref{as:gen2}.
Then,
\begin{enumerate}[(a)]
\item For any fixed $k\geq 0$ and fixed $\ell=1,\ldots,\Lm1$,
the parameter list $\Upsilon_{k\ell}^+$ converges as
\beq \label{eq:Lamplim}
    \lim_{N \arr \infty} \Upsilon_{k\ell}^+ = \Upsilonbar_{k\ell}^+
\eeq
almost surely.
Also, the rows of
$\wbf_\ell$, $\pbf^0_{\lm1}$, $\qbf^0_{\ell}$, $\pbf_{0,\lm1}^+,\ldots,\pbf_{k,\lm1}^+$ and $\qbf_{0\ell}^\pm,\ldots,\qbf_{k\ell}^\pm$
almost surely jointly converge empirically  with limits,
\beq \label{eq:PQplim}
    (\p^0_{\lm1},\p^+_{i,\lm1},\q^-_{j\ell},\q^0_{\ell},\q^+_{j\ell})  \xRightarrow{2}
        (P^0_{\lm1},P^+_{i,\lm1},Q^-_{j\ell},Q^0_{\ell}, Q^+_{j\ell}),
\eeq
for all $0\leq i,j\leq k$, where the variables
$P^0_{\lm1}$, $P_{i,\lm1}^+$ and $Q_{j\ell}^-$
are zero-mean jointly Gaussian random variables independent of $W_\ell$ and with covariance matrix given by
\begin{equation}\label{eq:PQpcorr}
\begin{aligned} 
\MoveEqLeft
\Cov(P^0_{\lm1},P_{i,\lm1}^+) = \Kbf_{i,\lm1}^+, \quad \Exp(Q_{j\ell}^-)^2 = \taubf_{j\ell}^-,\\
\MoveEqLeft \Exp(P_{i,\lm1}^{+\mathsf{T}}Q_{j\ell}^-)  = \zero,
    \quad \Exp(P^{0\mathsf{T}}_{\lm1} Q_{j\ell}^-)  = \zero,
\end{aligned}
\end{equation}
and $Q^0_\ell$, $Q^+_{j\ell}$ are the random variable in lines~\ref{line:Q0def}, \ref{line:qp_se_gen},\ie,
\beq \label{eq:Qpf}
\begin{aligned}
\MoveEqLeft    Q^0_\ell = f^0_\ell(P^0_{\lm1},W_{\ell}), \\
\MoveEqLeft    Q^+_{j\ell} =
    f^+_{j\ell}(P^0_{\lm1},P^+_{j,\lm1},Q^-_{j\ell},W_\ell,\Upsilonbar_{j\ell}^+).
\end{aligned}
\eeq
An identical result holds for $\ell=0$ with all the variables $\pbf_{i,\lm1}^+$ and $P_{i,\lm1}^+$ removed.

\item For any fixed $k \geq 1 $ and fixed $\ell=1,\ldots,\Lm1$,
the parameter lists $\Upsilon_{k\ell}^-$ converge as
\beq \label{eq:Lamnlim}
    \lim_{N \arr \infty} \Upsilon_{k\ell}^- = \Upsilonbar_{k\ell}^-
\eeq
almost surely.
Also, the rows of
$\wbf_\ell$, $\pbf^0_{\lm1}$, $\pbf_{0,\lm1}^\pm,\ldots,\pbf_{\km1,\lm1}^\pm$,  and $\qbf_{0\ell}^-,\ldots,\qbf_{k\ell}^-$
almost surely jointly converge empirically  with limits,
\beq \label{eq:PQnlim}
    (\pbf^0_{\lm1},\pbf^+_{i,\lm1},\qbf^-_{j\ell},\pbf^-_{j,\ell-1}) \xRightarrow{2}
        (P^0_{\lm1},P^+_{i,\lm1},Q^-_{j\ell},P_{j,\ell-1}^-),
\eeq
for all $0\leq i\leq \km1$ and $0\leq j\leq k$, where the variables
$P^0_{\lm1}$, $P_{i,\lm1}^+$ and $Q_{j\ell}^-$
are zero-mean jointly Gaussian random variables independent of $W_\ell$ and with covariance matrix given by equation \eqref{eq:PQpcorr}
and $P^-_{j\ell}$ is the random variable in line~\ref{line:pn_se_gen}:
\beq \label{eq:Pnf}
    P^-_{j\ell} = f^-_{j\ell}(P^0_{\lm1},P^+_{j-1,\lm1},
                Q^-_{j\ell},W_\ell,\Upsilonbar_{j\ell}^-).
\eeq
An identical result holds for $\ell=L$ with all the variables $\qbf_{j\ell}^-$ and $Q_{j\ell}^-$ removed.

For $k=0$, $\Upsilon_{01}^-\rightarrow \Upsilonbar_{01}^-$ almost surely, and the rows $\{(\wbf_{\ell,n:},\pbf_{\ell-1,n:}^0,\qbf_{j\ell,n:}^-)\}_{n=1}^N$ empirically converge to independent random variables $(W_\ell,P_{\ell-1}^0,Q_{0\ell}^-)$.
\end{enumerate}
\end{theorem}
\begin{proof}  \ref{app:proof_of_general_convergence} is dedicated to proving this result.
\end{proof}

\section{Proof of Theorem~\ref{thm:general_convergence}} \label{app:proof_of_general_convergence}

The proof proceeds using mathematical induction. It largely mimics the proof for the case of $d=1$ which were the convergence results in \cite[Thm. 5]{pandit2019inference}. However, in the case of $d>1$, we observe that several quantities which were scalars in proving \cite[Thm. 5]{pandit2019inference} are now matrices. Due to the non-commutativity of these matrix quantities, we re-state the whole prove, while modifying the requisite matrix terms appropriately.

\subsection{Overview of the Induction Sequence}

The proof is similar to that of \cite[Theorem 4]{rangan2019vamp},
which provides a SE analysis for VAMP on a single-layer network.
The critical challenge here is to extend that proof
to multi-layer recursions.
Many of the ideas in the two proofs are similar, so we highlight only the
key differences between the two.

Similar to the SE analysis of VAMP in \cite{rangan2019vamp},
we use an induction argument.  However, for the multi-layer proof,
we must index over both the iteration index $k$ and layer index $\ell$. To this end,
let $\mathcal{H}_{k\ell}^+$ and $\mathcal{H}_{k\ell}^-$ be the hypotheses:
\begin{itemize}
\item $\mathcal{H}_{k\ell}^+$:  The hypothesis that Theorem~\ref{thm:general_convergence}(a)
is true for a given $k$ and $\ell$, where $0\leq \ell\leq L-1$.
\item $\mathcal{H}_{k\ell}^-$:  The hypothesis that Theorem~\ref{thm:general_convergence}(b)
is true for a given $k$ and $\ell$, where $1\leq \ell\leq L$.
\end{itemize}
We prove these hypotheses by induction via a sequence of implications,
\beq \label{eq:induc}
\begin{aligned}
\MoveEqLeft
    \{\mc H^-_{0\ell}\}_{\ell=1}^L\cdots \Rightarrow \mathcal{H}_{k1}^- \Rightarrow \mathcal{H}_{k0}^+ \Rightarrow \cdots\\
    \MoveEqLeft\qquad\quad \Rightarrow  \mathcal{H}_{k,\Lm1}^+
    \Rightarrow \mathcal{H}_{\kp1,L}^- \Rightarrow \cdots \Rightarrow \mathcal{H}_{\kp1,1}^- \Rightarrow \cdots,
    \end{aligned}
\eeq
beginning with the hypotheses $\{\mathcal{H}^-_{0\ell}\}$ for all $\ell=1,\ldots,\Lm1$. 

\subsection{Base Case: Proof of \texorpdfstring{$\{\mc H_{0\ell}^-\}_{\ell=1}^L$}{H0l-}}
The base case corresponds to the hypotheses $\{\mc H_{0\ell}^-\}_{\ell=1}^L.$ Note that Theorem \ref{thm:general_convergence}(b) states that for $k=0$, we need $\Upsilon_{01}^-\rightarrow \Upsilonbar_{01}^-$ almost surely, and $\{(\wbf_{\ell,n:},\p_{\ell-1,n:}^0,\qbf_{j\ell,n:}^-)\}_{n=1}^N$ empirically converge to independent random variables $(W_\ell,P_{\ell-1}^0,Q_{0\ell}^-)$. These follow directly from equations \eqref{eq:qwinitlim} and \eqref{eq:Lambar01lim} in Assumption 1 (a).

\subsection{Inductive Step: Proof of \texorpdfstring{$\mc H_{k,\ell+1}^+$}{H0l+}}

Fix a layer index $\ell=1,\ldots,\Lm1$ and an iteration index $k=0,1,\ldots$. We show the implication $\cdots\implies \mc H^+_{k,\ell+1}$ in \eqref{eq:induc}. All other implications can be proven similarly using symmetry arguments.
\begin{definition}[Induction hypothesis] The hypotheses {prior} to $\mathcal{H}^+_{k,\lp1}$ in the sequence \eqref{eq:induc},
but not including $\mathcal{H}^+_{k,\lp1}$, are true.  
\end{definition}
The inductive step then corresponds to the following result.
\begin{lemma}\label{lem:pqconvinduc}
Under the induction hypothesis, $\mc H_{k,\ell+1}^+$ holds
\end{lemma}

Before proving the inductive step in Lemma \ref{lem:pqconvinduc}, we prove two intermediate lemmas. Let us start by defining some notation. Define $\Pbf_{k\ell}^+ := \left[ \pbf_{0\ell}^+ \cdots \pbf_{k\ell}^+ \right] \in \R^{N_\ell \times (\kp1)d},
$ be a matrix whose column blocks are the first $\kp1$ values of the matrix $\pbf^+_{\ell}$. We define the matrices $\Pbf_{k\ell}^-$, $\Qbf_{k\ell}^+$ and $\Qbf_{k\ell}^-$ in a similar manner with values of $\pbf_{\ell}^-,\qbf_\ell^+$ and $\qbf_\ell^-$ respectively.

Note that except the initial matrices $\{\wbf_\ell,\qbf_{0\ell}^-\}_{\ell=1}^L$, all later iterates in Algorithm \ref{algo:gen} are random due to the randomness of $\V_\ell$.
Let $\Gset_{k\ell}^\pm$ denote the collection of random variables associated with the hypotheses,
$\mathcal{H}^{\pm}_{k\ell}$.  That is, for $\ell=1,\ldots,\Lm1$,
\begin{align*}
\label{eq:Gsetdef}
    \MoveEqLeft\Gset_{k\ell}^+ := \left\{ \wbf_{\ell},\pbf^0_{\lm1},\Pbf^+_{k,\lm1},\qbf^0_\ell,\Qbf^-_{k\ell},\Qbf_{k\ell}^+ \right\}, \\
    \MoveEqLeft\Gset_{k\ell}^- := \left\{ \wbf_{\ell},\pbf^0_{\lm1},\Pbf^+_{\km1,\lm1},\qbf^0_\ell,
        \Qbf^-_{k\ell},\Pbf^-_{k,\lm1} \right\}.
\end{align*}
For $\ell=0$ and $\ell=L$ we set, $
    \Gset_{k0}^+ := \left\{ \wbf_{0},\Qbf^-_{k0},\Qbf_{k0}^+ \right\}, \quad
    \Gset_{kL}^- := \left\{ \wbf_L,\pbf^0_{\Lm1},\Pbf^+_{\km1,\Lm1},\Pbf^-_{k,\Lm1} \right\}.$
    
Let $\Gsetbar_{k\ell}^+$ be the sigma algebra generated by the union of all the sets $\Gset_{k'\ell'}^\pm$
as they have appeared in the sequence \eqref{eq:induc} up to and including the final
set $\Gset_{k\ell}^+$.
Thus, the sigma algebra $\Gsetbar_{k\ell}^+$ contains all \textit{information} produced by Algorithm~\ref{algo:gen}
immediately \emph{before} line~\ref{line:pp_gen} in layer $\ell$ of iteration $k$. Note also that the random variables in Algorithm \ref{algo:gen_se} immediately before defining $P_{k,\ell}^+$ in line \ref{line:pp_se_gen} are all $\Gsetbar_{k\ell}^+$ measurable.

Observe that the matrix $\Vbf_\ell$ in Algorithm~\ref{algo:gen}
appears only during matrix-vector multiplications in lines~\ref{line:pp_gen} and \ref{line:pn_gen}.
If we define the matrices,
$
    \Abf_{k\ell} := \left[ \pbf^0_\ell, \Pbf_{\km1,\ell}^+ ~ \Pbf_{k\ell}^- \right], \quad
    \Bbf_{k\ell} := \left[ \qbf^0_\ell, \Qbf_{\km1,\ell}^+ ~ \Qbf_{k\ell}^- \right],
$
all the matrices in the set $\Gsetbar_{k\ell}^+$ will be unchanged for all
matrices $\Vbf_\ell$ satisfying the linear constraints
\beq \label{eq:ABVconk}
    \Abf_{k\ell} = \Vbf_\ell\Bbf_{k\ell}.
\eeq
Hence, the conditional distribution of $\Vbf_\ell$ given $\Gsetbar_{k\ell}^+$ is precisely
the uniform distribution on the set of orthogonal matrices satisfying
\eqref{eq:ABVconk}.  The matrices $\Abf_{k\ell}$ and $\Bbf_{k\ell}$ are of dimensions
$N_\ell \times (2k+2)d$.
From \cite[Lemmas 3,4]{rangan2019vamp}, this conditional distribution is given by
\beq \label{eq:Vconk}
    \left. \Vbf_\ell \right|_{\Gsetbar_{k\ell}^+} \eqd
    \Abf_{k\ell}(\Abf\tran_{k\ell}\Abf_{k\ell})^{-1}\Bbf_{k\ell}\tran + \Ubf_{\Abf_{k\ell}^\perp}\wt{\Vbf}_\ell\Ubf_{\Bbf_{k\ell}^\perp}\tran,
\eeq
where $\Ubf_{\Abf_{k\ell}^\perp}$ and $\Ubf_{\Bbf_{k\ell}^\perp}$ are $N_\ell \times (N_\ell-(2k+2)d)$ matrices
whose columns are an orthonormal basis for $\Range(\Abf_{k\ell})^\perp$ and $\Range(\Bbf_{k\ell})^\perp$.
The matrix $\wt{\Vbf}_\ell$ is  Haar distributed on the set of $(N_\ell-(2k+2)d)\times (N_\ell-(2k+2)d)$
orthogonal matrices and is independent of $\Gsetbar_{k\ell}^+$.

Next, similar to the proof of \cite[Thm. 4]{rangan2019vamp},
we can use \eqref{eq:Vconk} to write the conditional distribution of $\pbf_{k\ell}^+$ (from line~\ref{line:pp_gen} of Algorithm \ref{algo:gen}) given $\Gsetbar_{k\ell}^+$ as a sum of two terms
\begin{subequations}
\begin{align}
    \label{eq:ppart}
    \pbf_{k\ell}^+|_{\Gsetbar_{k\ell}^+} &= \Vbf_\ell|_{\Gsetbar_{k\ell}^+}\ \qbf_{k\ell}^+ \overset{d}= \pbf_{k\ell}^{\rm +det} + \pbf_{k\ell}^{\rm +ran},\\
    \label{eq:pdet}
    \pbf_{k\ell}^{\rm +det} &:= \Abf_{k\ell}(\Bbf\tran_{k\ell}\Bbf_{k\ell})^{-1}\Bbf_{k\ell}\tran\qbf_{k\ell}^+\\
    \label{eq:pran}
    \pbf_{k\ell}^{\rm +ran} &:= \Ubf_{\Bbf_k^\perp}\wt{\Vbf}_\ell\tran \Ubf_{\Abf_k^\perp}\tran \qbf_{k\ell}^+.
\end{align}
\end{subequations}
where we call $\pbf_{k\ell}^{\rm +det}$ the \emph{deterministic} term and
$\pbf_{k\ell}^{\rm +ran}$ the \emph{random} term. The next two lemmas characterize the limiting distributions
of the deterministic and random terms.

\begin{lemma} \label{lem:pconvdet}
Under the induction hypothesis, the rows of the ``deterministic" term
$\pbf_{k\ell}^{+\rm det}$ along with the rows
of the matrices in $\Gsetbar_{k\ell}^+$  converge empirically.
In addition, there exists constant $d\times d$ matrices $\beta_{0\ell}^+,\ldots,\beta^+_{\km1,\ell}$ such that
\beq \label{eq:pconvdet}
\begin{aligned}
   \pbf_{k\ell}^{\rm +det}\xRightarrow{2} P_{k\ell}^{\rm +det}
    :=  P^0_\ell\beta^0_\ell +  \sum_{i=0}^{\km1} P_{i\ell}^+\beta_{i\ell},
\end{aligned}\eeq
where $P_{k\ell}^{+\rm det} \in\Real^{1\times d}$ is the limiting random vector for the rows of $\pbf_{k\ell}^{\rm det}$.
\end{lemma}
\begin{proof}
The proof is similar that of \cite[Lem. 6]{rangan2019vamp}, but we go over the details
as there are some important differences in the multi-layer matrix case.
Define
$ \label{eq:PQaug}
    \wt{\Pbf}_{\km1,\ell}^+ = \left[ \pbf^0_\ell, ~ \Pbf_{\km1,\ell}^+ \right],
    \wt{\Qbf}_{\km1,\ell}^+ = \left[ \qbf^0_\ell, ~ \Qbf_{\km1,\ell}^+ \right],
$
which are the matrices in $\Real^{N_\ell\times (k+1)d}$.
We can then write $\Abf_{k\ell}$ and $\Bbf_{k\ell}$ from \eqref{eq:ABVconk} as
\beq \label{eq:ABdef2}
    \Abf_{k\ell} := \left[ \wt{\Pbf}_{\km1,\ell}^+ ~ \Pbf_{k\ell}^- \right], \quad
    \Bbf_{k\ell} := \left[ \wt{\Qbf}_{\km1,\ell}^+ ~ \Qbf_{k\ell}^- \right],
\eeq
We first evaluate the asymptotic values of various terms in \eqref{eq:pdet}.
By definition of $\Bbf_{k\ell}$ in \eqref{eq:ABdef2},
\[
    \Bbf\tran_{k\ell}\Bbf_{k\ell} = \begin{bmatrix}
        (\wt{\Qbf}_{\km1,\ell}^+)\tran\wt{\Qbf}_{\km1,\ell}^+ & (\wt{\Qbf}_{\km1,\ell}^+)\tran\Qbf_{k\ell}^- \\
        (\Qbf_{k\ell}^-)\tran\wt{\Qbf}_{\km1,\ell}^+ & (\Qbf_{k\ell}^-)\tran\Qbf_{k\ell}^-
        \end{bmatrix}
\]
We can then evaluate the asymptotic values of these
terms as follows:  For $0\leq i,j\leq k-1$ the asymptotic value of the
$(i+2,j+2)^{\rm nd}$ $d\times d$ block of the matrix $(\wt{\Qbf}_{\km1,\ell}^+)\tran\wt{\Qbf}_{\km1,\ell}^+$ is
\begin{equation*}
\begin{aligned}
    \MoveEqLeft \lim_{N \arr \infty} \tfrac{1}{N_\ell} \left[ (\wt{\Qbf}_{\km1,\ell}^+)\tran\wt{\Qbf}_{\km1,\ell}^+ \right]_{i+2,j+2}
        \stackrel{(a)}{=} \lim_{N \arr \infty}
        \frac{1}{N_\ell} (\qbf_{i\ell}^+)\tran\qbf_{j\ell}^+ \\
        &= \lim_{N \arr \infty} \tfrac{1}{N_\ell} \sum_{n=1}^{N_\ell} [\qbf_{i\ell}^+]_{n:}[\qbf_{j\ell}^+]_{n:}\tran
        \stackrel{(b)}{=} \Exp\left[ Q_{i\ell}^{+\mathsf{T}}Q_{j\ell}^{+} \right]
\end{aligned}
\end{equation*}
where (a) follows since the $(i+2)^{\rm th}$ column block of $\wt{\Qbf}_{\km1,\ell}^+$
is $\qbf_{i\ell}^+$, and (b) follows due to the empirical convergence assumption in \eqref{eq:PQplim}.
Also, since the first column block of $\wt{\Qbf}_{\km1,\ell}^+$ is $\qbf^0_\ell$,
we obtain that
\beq\begin{aligned}
    \MoveEqLeft
    \lim_{N_\ell \arr \infty}  \tfrac{1}{N_\ell}
        (\wt{\Qbf}_{k-1,\ell}^+)\tran\wt{\Qbf}_{k-1,\ell}^+ = \Rbf^+_{k-1,\ell}\qquad{\rm and}\\
        \MoveEqLeft \lim_{N_\ell \arr \infty}  \tfrac{1}{N_\ell}  (\Qbf_{k\ell}^-)\tran\Qbf_{k\ell}^- = \Rbf^-_{k\ell},
\end{aligned}\eeq
where $\Rbf^+_{k-1,\ell}\in\Real^{(k+1)d\times(k+1)d}$ is the covariance matrix of
$\left[Q^0_\ell\ Q_{0\ell}^+\ \ldots\ Q_{k-1,\ell}^+\right]$,
and $\Rbf^-_{k\ell} \in\Real^{(k+1)d\times(k+1)d}$ is the covariance matrix of 
$\left[Q_{0\ell}^-\ Q_{1\ell}^-\ \ldots\ Q_{k\ell}^-\right]$.
For the matrix $(\wt\Qbf_{\km1,\ell}^+)\tran\Qbf_{k\ell}^-$,
first observe that the limit of the divergence free condition \eqref{eq:fdivfree} implies
\beq\label{eq:fpdivfree}
\begin{aligned}
 \MoveEqLeft   \Exp\left[ \frac{\partial f_{i\ell}^+(P_{i,\lm1}^+,Q_{i\ell}^-,W_\ell,\Upsilonbar_{i\ell})}{\partial Q_{i\ell}^-} \right]\\
\MoveEqLeft    = \lim_{N_\ell \arr \infty}  \bkt{\frac{\partial \fbf^+_{i\ell}(\pbf^+_{i,\lm1},\qbf_{i\ell}^-,\wbf_\ell,\Upsilonbar^+_{i\ell})}{
        \partial \qbf_{i\ell}^-} }  = \zero,
\end{aligned}
\eeq
for any $i$.  Also, by the induction hypothesis $\mathcal{H}_{k\ell}^+$,
\beq \label{eq:pqxcorrpf}
    \Exp(P_{i,\lm1}^{+\msf{T}}Q_{j\ell}^-)  = \zero, \quad
    \Exp(P_{\lm1}^{0\msf{T}} Q_{j\ell}^-) = \zero,
\eeq
for all $0\leq i,j \leq k$.
Therefore using \eqref{eq:Qpf}, the cross-terms $\Exp(Q_{i\ell}^{+\msf{T}}Q_{j\ell}^-)$ are given by
\beq\label{eq:Qijstein}\begin{aligned}
    \MoveEqLeft\Exp(f_{i\ell}^+(P^0_{\lm1},P_{i,\lm1}^+,Q_{i\ell}^-,W_\ell,\Upsilonbar_{i\ell})^{\msf{T}}Q_{j\ell}^-) \\
    \MoveEqLeft\stackrel{(a)}{=}  \Exp\left[ \tfrac{\partial f_{i\ell}^+(P^0_{\lm1},P_{i,\lm1}^+,Q_{i\ell}^-,W_\ell,\Upsilonbar^+_{i\ell})}
        {\partial P_{\lm1}^0} \right]
        \Exp(P_{\lm1}^{0\msf{T}}Q_{j\ell}^-)     \\
     \MoveEqLeft+\Exp\left[ \tfrac{\partial f_{i\ell}^+(P^0_{\lm1},P_{i,\lm1}^+,Q_{i\ell}^-,W_\ell,\Upsilonbar^+_{i\ell})}{\partial P_{i,\lm1}^+} \right]
        \Exp(P_{i,\lm1}^{+\msf{T}}Q_{j\ell}^-)\\
     \MoveEqLeft+ \Exp\left[ \tfrac{\partial f_{i\ell}^+(P^0_{\lm1},P_{i,\lm1}^+,Q_{i\ell}^-,W_\ell,\Upsilonbar^+_{i\ell})}
        {\partial Q_{i\ell}^-} \right]
        \Exp(Q_{i\ell}^{-\msf{T}}Q_{j\ell}^-)
     \stackrel{(b)}{=} \zero, 
\end{aligned}
\eeq
(a) follows from a multivariate version of Stein's Lemma \cite[eqn.(2)]{liu1994siegel}; and (b) follows from \eqref{eq:fpdivfree}, and \eqref{eq:pqxcorrpf}.
Consequently,
\beq\label{eq:BBlim_and_Bqlim}
\begin{aligned} 
    \MoveEqLeft\lim_{N_\ell \arr \infty} \tfrac{1}{N_\ell} \Bbf\tran_{k\ell}\Bbf_{k\ell} = \begin{bmatrix}
        \Rbf_{\km1,\ell}^+ & \zero \\
        \zero & \Rbf_{k\ell}^-
        \end{bmatrix} ,\quad{\rm and}\\
        \MoveEqLeft
     \lim_{N_\ell \arr \infty} \tfrac{1}{N_\ell} \Bbf_{k\ell}\tran\qbf_{k\ell}^+= \
    \begin{bmatrix} \bbf^+_{k\ell} \\ \zero \end{bmatrix} ,
\end{aligned}
\eeq
where $\bbf^+_{k\ell} := \left[\Exp(Q_{0\ell}^{+\msf{T}}Q_{k\ell}^+) ~ \Exp(Q_{1\ell}^{+\msf{T}}Q_{k\ell}^+)
        ~\cdots~ \Exp(Q_{\km1,\ell}^{+\msf{T}}Q_{k\ell}^+) \right]\tran,$ is the matrix of correlations. We again have $\zero$ in the second term because $\Exp[Q_{i\ell}^{+\msf{T}}Q_{j\ell}^-]=\zero$ for all $0\leq i,j\leq k$. Hence we have
\beq \label{eq:Bqmult}
    \lim_{N_\ell \arr \infty} (\Bbf\tran_{k\ell}\Bbf_{k\ell})^{-1}\Bbf_{k\ell}\tran\qbf_{k\ell}^+ =
    \begin{bmatrix}  \betabf_{k\ell}^+ \\ \mathbf{0} \end{bmatrix}, \ \  \betabf_{k\ell}^+ := \begin{bmatrix} \Rbf^+_{\km1,\ell} \end{bmatrix}^{-1}\bbf^+_{k\ell}.
\eeq
Therefore, $\pbf_{k\ell}^{+\rm det}$ equals
\beq\begin{aligned}
    \MoveEqLeft\Abf_{k\ell}(\Bbf\tran_{k\ell}\Bbf_{k\ell})^{-1}\Bbf_{k\ell}\tran\qbf_{k\ell}^+
    = \left[ \wt{\Pbf}_{\km1,\ell}^+ ~ \Pbf_{k,\ell}^- \right]
\begin{bmatrix}\betabf_{k\ell}^+  \\ \zero \end{bmatrix}
    + O\left(\tfrac{1}{N_\ell}\right)\\
    \MoveEqLeft=  \pbf^0_\ell\beta^0_\ell +
    \sum_{i=0}^{\km1} \pbf_{i\ell}^+\beta_{i\ell}^+ + O\left(\tfrac{1}{N_\ell}\right),
\end{aligned}
\eeq
where $\beta^0_\ell$ and $\beta_{i\ell}^+$ are $d\times d$ block matrices of $\betabf_{k\ell}^+$ and
the term $O(\tfrac1{N_\ell})$ means a matrix sequence, ${\boldsymbol\varphi}(N) \in \R^{N_{\ell}}$ such that $\lim_{N \arr\infty} \tfrac{1}{N} \|{\boldsymbol\varphi}(N)\|^2 = 0.$
A continuity argument then shows the empirical convergence \eqref{eq:pconvdet}.
\end{proof}

\begin{lemma} \label{lem:pconvran}
Under the induction hypothesis, the components of
the ``random" term $\pbf_{k\ell}^{+\rm ran}$ along with the components
of the vectors in $\Gsetbar_{k\ell}^+$ almost surely converge empirically.
The components of $\pbf_{k\ell}^{+\rm ran}$ converge as
\beq \label{eq:pconvran}
     \pbf_{k\ell}^{+\rm ran} \xRightarrow{2} U_{k\ell},
\eeq
where $U_{k\ell}$ is a zero mean Gaussian random vector in $\Real^{1\times d}$
independent of the limiting random variables corresponding to the variables
in $\Gsetbar_{k\ell}^+$.
\end{lemma}
\begin{proof}
The proof is identical to that of \cite[Lemmas 7,8]{rangan2019vamp}.
\end{proof}

We are now ready to prove Lemma \ref{lem:pqconvinduc}.

\old{
\begin{lemma} \label{lem:pqconvinduc}
Under the induction hypothesis, the parameter list $\Upsilon_{k,\lp1}^+$ almost surely converges as
\beq \label{eq:Lampliminduc}
    \lim_{N_{\lp1} \arr \infty} \Upsilon_{k,\lp1}^+ = \Upsilonbar_{k,\lp1}^+,
\eeq
where $\Upsilonbar_{k,\lp1}$ is the parameter list generated from the SE recursion, Algorithm~\ref{algo:gen_se}.
Also, the components of
$\wbf_{\lp1}$, $\pbf^0_{\ell}$, $\qbf^0_{\lp1}$, $\pbf_{0,\ell}^+,\ldots,\pbf_{k,\ell}^+$ and $\qbf_{0,\lp1}^\pm,\ldots,\qbf_{k,\lp1}^\pm$
almost surely empirically converge jointly with limits,
\beq \label{eq:PQpliminduc}
    \lim_{N \arr \infty} \left\{
        (p^0_{\ell,n},p^+_{i\ell,n},q^0_{\lp1,n},q^-_{j,\lp1,n},q^+_{j,\lp1,n}) \right\} =
        (P^0_{\ell},P^+_{i\ell},Q^0_{\lp1},Q^-_{j,\lp1}, Q^+_{j,\lp1}),
\eeq
for all $i,j=0,\ldots,\kp1$, where the variables
\beq \label{eq:pqvecinduc}
    (P^0_\ell,P_{0\ell}^+,\ldots,P_{k,\ell}^+,Q_{0,\lp1}^-,\ldots,Q_{k,\lp1}^-),
\eeq
are zero-mean jointly Gaussian random variables independent of $W_\ell$ with
\beq \label{eq:PQpcorrinduc}
    \Cov(P^0_{\ell},P_{i,\ell}^+) = \Kbf_{i\ell}^+, \quad \Exp(Q_{j,\lp1}^-)^2 = \tau_{j,\lp1}^-, \quad \Exp(P_{i,\ell}^+Q_{j,\lp1}^-)  = 0,
    \quad \Exp(P^0_{\ell}Q_{j,\lp1}^-)  = 0,
\eeq
and $Q^0_{\lp1}$ and $Q^+_{j,\lp1}$ are the random variables in line~\ref{line:qp_se_gen}:
\beq \label{eq:Qpfinduc}
    Q^0_{\lp1} = f^0_{\lp1}(P^0_{\ell},W_{\lp1}), \quad
    Q^+_{j,\lp1} =
    f^+_{j,\lp1}(P^0_{\ell},P^+_{i\ell},Q^-_{j,\lp1},W_{\lp1},\Upsilonbar_{k,\lp1}^+).
\eeq
\end{lemma}
}

\begin{proof}[Proof of Lemma \ref{lem:pqconvinduc}]
Using the partition \eqref{eq:ppart} and Lemmas~\ref{lem:pconvdet} and \ref{lem:pconvran},
we see that the components of the
vector sequences in $\Gsetbar_{k\ell}^+$ along with $\pbf^+_{k\ell}$
almost surely converge jointly empirically, where the components of $\pbf^+_{k\ell}$
have the limit
\beq \label{eq:pklim}
\begin{aligned}
    \pbf^+_{k\ell} 
    = \pbf^{\rm det}_{k\ell} + \pbf^{\rm ran}_{k\ell}
    \xRightarrow{2}   P^0_\ell\beta^0_\ell + \sum_{i=0}^{\km1}   P_{i\ell}^+\beta_{i\ell}^+ + U_{k\ell} =: P_{k\ell}^+.
\end{aligned}
\eeq
{Note that the above Wasserstein-2 convergence can be shown using the same arguments involved in showing that if $X_N|\mc F\overset{d}{\implies} X|\mc F,$ and $Y_N|\mc F\overset{d}{\implies} c,$ then $(X_N,Y_N)|\mc F\overset{d}{\implies} (X,c)|\mc F$ for some constant $c$ and sigma-algebra $\mc F$.}

We first establish the Gaussianity of $P_{k\ell}^+$. Observe that by the induction hypothesis, $\mathcal{H}_{k,\lp1}^-$ holds whereby $(P_\ell^0,P_{0\ell}^+,\ldots,P_{\km1,\ell}^+,Q_{0,\lp1}^-,\ldots,Q_{k,\lp1}^-),$
is jointly Gaussian. Since $U_k$ is Gaussian and independent of $(P_\ell^0,P_{0\ell}^+,\ldots,P_{k-1,\ell}^+,Q_{0,\lp1}^-,\ldots,Q_{k,\lp1}^-),$ we can conclude from  \eqref{eq:pklim} that
$(P_\ell^0,P_{0\ell}^+,\ldots,P_{\km1,\ell}^+,P_{k\ell}^+,Q_{0,\lp1}^-,\ldots,Q_{k,\lp1}^-)$ is jointly Gaussian.

We now need to prove the correlations of this jointly Gaussian random vector are as claimed by $\mc H_{k,\ell+1}^+$.  Since $\mathcal{H}_{k,\lp1}^-$ is true, we know
that \eqref{eq:PQpcorr} is true for
all $i=0,\ldots,\km1$ and $j=0,\ldots,k$ and $\ell=\ell+1$.  Hence,
we need only to prove the additional identity for $i=k$,
namely the equations:
$\Cov(P^0_\ell,P_{k\ell}^+)^2 = \Kbf_{k\ell}^+
    $ and $
    \Exp(P_{k\ell}^+Q_{j,\lp1}^-) = 0.
$
First observe that
\begin{align*}
\MoveEqLeft
    \Exp(P_{k\ell}^{+\msf{T}}P_{k\ell}^+)^2  \stackrel{(a)}{=} \lim_{N_\ell \arr \infty} \tfrac{1}{N_\ell}
        \pbf_{k\ell}^{+\msf{T}}\pbf_{k\ell}^{+}\\
        \MoveEqLeft
          \stackrel{(b)}{=} \lim_{N_\ell \arr \infty} \tfrac{1}{N_\ell}
        \qbf_{k\ell}^{+\msf{T}}\qbf_{k\ell}^+ \stackrel{(c)}{=}  \Exp\left( Q_{k\ell}^{+\msf{T}}Q_{k\ell}^+ \right)^2
\end{align*}
where (a) follows from the fact that the rows of $\pbf^+_{k\ell}$ converge empirically
to $P_{k\ell}^+$;
(b) follows from line \ref{line:pp_gen} in Algorithm~\ref{algo:gen} and the fact that $\Vbf_\ell$ is orthogonal;
and
(c) follows from the fact that the rows of $\qbf^+_{k\ell}$ converge empirically
to $Q_{k\ell}^+$ from hypothesis $\mc H_{k,\ell}^+$.  Since $\pbf^0_\ell = \Vbf_\ell \qbf^0$, we similarly obtain that
$
    \Exp(P^{0\msf{T}}_\ell P_{k\ell}^+) = \Exp(Q^{0\msf{T}}_\ell Q_{k\ell}^+), \quad
    \Exp(P^{0\msf{T}}_\ell P^0_\ell )= \Exp(Q^{0\msf{T}}_\ell Q^0_\ell),
$
from which we conclude
\beq \label{eq:PQcorr3}
    \Cov(P^0_\ell, P_{k\ell}^+) = \Cov(Q^0_\ell, Q_{k\ell}^+) =: \Kbf^+_{k\ell},
\eeq
where the last step follows from the definition of $\Kbf^+_{k\ell}$ in line~\ref{line:pp_se_gen} of Algorithm \ref{algo:gen_se}.
Finally, we observe that for $0\leq j\leq k$
\beq \label{eq:PQcorr4}
\begin{aligned}
    \MoveEqLeft
    \Exp(P_{k\ell}^{+\msf{T}}Q_{j,\lp1}^-) \stackrel{(a)}{=}
     \beta^{0\msf{T}}_\ell\Exp(P_{\ell}^{0\msf{T}}Q_{j,\lp1}^-)\\
     \MoveEqLeft\qquad
     + \sum_{i=0}^{\km1} \beta_{i\ell}^{+\msf{T}} \Exp(P_{i\ell}^{+\msf{T}}Q_{j,\lp1}^-)
        + \Exp(U_{k\ell}^{\msf{T}}Q_{j,\lp1}^-) \stackrel{(b)}{=} \zero,
        \end{aligned}
\eeq
where (a) follows from \eqref{eq:pklim} and, in (b), we used the fact that
$\Exp(P_{\ell}^{0\msf{T}}Q_{j,\lp1}^-) = \zero$ and
$\Exp(P_{i\ell}^{+\msf{T}}Q_{j,\lp1}^-) = \zero$ since \eqref{eq:PQpcorr} is true for $i\leq \km1$ corresponding to $\mc H_{k,\ell+1}^-$ and
$\Exp(U_{k\ell}^{\msf{T}}Q_{j,\lp1}^-) = \zero$ since $U_{k\ell}$ is independent of $\Gsetbar_{k\ell}^+$, and $Q_{j,\lp1}^-$ is $\Gsetbar_{k\ell}^+$ measurable.
Thus, with \eqref{eq:PQcorr3} and \eqref{eq:PQcorr4}, we have proven all the correlations in
\eqref{eq:PQpcorr} corresponding to $\mc H_{k,\ell+1}^+$.

Next, we prove the convergence of the parameter lists $\Upsilon_{k,\ell+1}^+$ to $\Upsilonbar_{k,\ell+1}^+$.  Since $\Upsilon^+_{k\ell} \arr \Upsilonbar_{k\ell}^+$ due to hypothesis $\mc H_{k\ell}^+$,
and $\varphi_{k,\lp1}^+(\cdot)$ is uniformly Lipschitz continuous,
we have that $\lim_{N \arr \infty} \mu^+_{k,\lp1}$ from line~\ref{line:mup_gen} in Algorithm~\ref{algo:gen}
converges almost surely as
\beq\begin{aligned}
    \MoveEqLeft  \lim_{N \arr \infty}
        \bkt{\varphibf_{k,\lp1}^+(\pbf^0_\ell,\pbf^+_{k\ell},\qbf_{k,\lp1}^-,\wbf_{\lp1},\Upsilonbar_{k\ell}^+)}\\
    \MoveEqLeft\qquad=    \Exp\left[ \varphi_{k,\lp1}^+(P^0_\ell,P^+_{k\ell},Q_{k,\lp1}^-,W_{\lp1},\Upsilonbar_{k\ell}^+)
        \right]  = \mubar^+_{k,\lp1},
\end{aligned}
\eeq
where $\mubar^+_{k,\lp1}$ is the value in line~\ref{line:mup_se_gen} in Algorithm~\ref{algo:gen_se}.
Since $T^+_{k,\lp1}(\cdot)$ is continuous, we have that $\lambda_{k,\lp1}^+$ in
line~\ref{line:lamp_gen} in Algorithm~\ref{algo:gen} converges as
$
    \lim_{N \arr \infty} \lambda_{k,\lp1}^+ 
= T_{k,\lp1}^+(\mubar_{k,\lp1}^+,\Upsilonbar_{k\ell}^+) =: \lambdabar_{k,\lp1}^+,
$
from
line~\ref{line:lamp_se_gen} in Algorithm~\ref{algo:gen_se}. Therefore, we have the limit
\beq\label{eq:Lampliminduc}
    \lim_{N \arr \infty} \Upsilon_{k,\lp1}^+ =
    \lim_{N \arr \infty} (\Upsilon_{k,\ell}^+,\lambda_{k,\lp1}^+)
    = (\Upsilonbar_{k,\ell}^+,\lambdabar_{k,\lp1}^+) = \Upsilonbar_{k,\lp1}^+,
\eeq
which proves the convergence of the parameter lists stated in $\mc H_{k,\ell+1}^+$.
Finally, using \eqref{eq:Lampliminduc}, the empirical convergence of the matrix
sequences $\pbf^0_\ell$, $\pbf_{k\ell}^+$ and $\qbf_{k,\lp1}^-$ and the uniform Lipschitz continuity of the update function $f_{k,\lp1}^+(\cdot)$ we obtain that $\qbf_{k,\lp1}^+ $ equals
\begin{align*}
    \MoveEqLeft\fbf_{k,\lp1}^+(\pbf^0_{\ell},\pbf_{k\ell}^-, \qbf_{k,\lp1}^-,\wbf_{\lp1},\Upsilon_{k,\lp1}^+)\\
    \MoveEqLeft\qquad\qquad\xRightarrow{2} f_{k,\lp1}^+(P^0_\ell,P_{k\ell}^-, Q_{k,\lp1}^-,W_{\lp1},\Upsilonbar_{k,\lp1}^+) =: Q^+_{k,\lp1},
\end{align*}
which proves the claim \eqref{eq:Qpf} for $\mc H_{k,\ell+1}^+$.  This completes the proof.
\end{proof}

An overview of the iterates in Algorithm \ref{algo:gen} is depicted in (TOP) and (MIDDLE) of Figure \ref{fig:mlmatvamp_equivalent_system}. Theorem \ref{thm:general_convergence} shows that the rows of the iterates of Algorithm \ref{algo:gen} converge empirically with $2^{\rm nd}$ order moments to random variables defined in Algorithm \ref{algo:gen_se}. The random variables defined in Algo. \ref{algo:gen_se} are depicted in Figure \ref{fig:mlmatvamp_equivalent_system} (BOTTOM).

 }

\newpage
\bibliographystyle{IEEEtran}
\bibliography{bibl}

\end{document}